\documentclass[11pt]{article}
\ifdefined\pdfsuppressptexinfo\pdfsuppressptexinfo=-1\fi
\newif\ifglancepublic
\glancepublictrue 
\ifglancepublic\usepackage[preprint]{acl}\else\usepackage[review]{acl}\fi
\usepackage{times}
\usepackage{latexsym}
\usepackage[T1]{fontenc}
\usepackage[utf8]{inputenc}
\usepackage{microtype}
\usepackage{inconsolata}
\usepackage{graphicx}

\usepackage{amsmath,amssymb,amsthm}
\usepackage{booktabs}
\usepackage{colortbl}
\usepackage{algorithm}
\usepackage{algpseudocode}

\theoremstyle{definition}
\newtheorem{theorem}{Theorem}
\newtheorem{proposition}{Proposition}

\newtheorem{lemma}{Lemma}
\newtheorem{assumption}{Assumption}
\newtheorem{definition}{Definition}
\newtheorem*{thmlosslessfull}{Theorem~\ref{thm:lossless} (full)}
\newtheorem*{thmtailfull}{Theorem~\ref{thm:tail} (full)}
\newtheorem*{corgroundedfull}{Corollary (Grounded is more draftable)}

\newcommand{\E}{\mathbb{E}}
\newcommand{\Prob}{\mathbb{P}}
\newcommand{\Var}{\mathrm{Var}}
\newcommand{\Hop}{\mathbb{H}}
\newcommand{\logit}{\operatorname{logit}}
\newcommand{\sigm}{\sigma}
\DeclareMathOperator*{\argmax}{arg\,max}
\newcommand{\HH}{H}          
\newcommand{\acc}{a}         
\newcommand{\pmatch}{p_{\mathrm{m}}}   
\newcommand{\pj}{p_{\mathrm{m},j}}    
\newcommand{\Eacc}{\E[\acc\mid\HH]}
\newcommand{\Lmax}{L}        
\newcommand{\Tree}{S}        
\newcommand{\budget}{N}      
\newcommand{\Acc}{A}         

\newcommand{\fitcol}[1]{\resizebox{\ifdim\width>\columnwidth\columnwidth\else\width\fi}{!}{#1}}
\newcommand{\fittext}[1]{\resizebox{\ifdim\width>\textwidth\textwidth\else\width\fi}{!}{#1}}

\definecolor{band}{gray}{0.915}
\definecolor{oursrow}{RGB}{226,237,250}
\newcommand{\grp}[2]{\rowcolor{band}\multicolumn{#1}{@{}c@{}}{\textit{#2}\strut}\\}
\newcommand{\g}[1]{\cellcolor{oursrow}#1}

\title{Vision Is Not Overhead: One-Pass Block Drafting for Lossless Speculative Decoding in Vision-Language Models}
\ifglancepublic
\author{Jungseob Lee\textsuperscript{1} \quad Seongtae Hong\textsuperscript{1} \quad Dongyub Jude Lee\textsuperscript{2} \quad Chanjun Park\textsuperscript{3} \\
\bfseries Jaehyung Seo\textsuperscript{4} \quad Sugyeong Eo\textsuperscript{5}\thanks{Corresponding authors.} \quad Heuiseok Lim\textsuperscript{1}\footnotemark[1] \\[4pt]
{\normalsize\mdseries \textsuperscript{1}Korea University \quad \textsuperscript{2}Zoom Communications \quad \textsuperscript{3}Soongsil University} \\
{\normalsize\mdseries \textsuperscript{4}Konkuk University \quad \textsuperscript{5}Yonsei University} \\[2pt]
{\small\mdseries \texttt{\{omanma1928,ghdchlwls123,limhseok\}@korea.ac.kr} \quad \texttt{jude.lee@zoom.us}} \\
{\small\mdseries \texttt{chanjun.park@ssu.ac.kr} \quad \texttt{seojae777@konkuk.ac.kr} \quad \texttt{s.eo@yonsei.ac.kr}}}
\else
\author{Anonymous ACL submission}
\fi

\begin{document}
\maketitle

\begin{abstract}
Speculative decoding accelerates generation without changing its output, but on vision-language models (VLMs) a self-reinforcing cycle holds it back. Because an autoregressive drafter pays a sequential pass for each drafted token, it must stay small and can ill afford to attend to the image at each pass. Prior work therefore compresses or hides the image, leaving the drafter weakest on the text the image determines. We present GLANCE, a one-pass block drafter that breaks this cycle on an unmodified VLM target. Its block-diffusion head drafts a whole block in one forward pass over the target's already fused vision-language states, reading the multimodal context once, however deep the draft. The target verifies a wide candidate tree in one pass and commits exactly its greedy output. In one production engine at a fixed round budget, GLANCE decodes up to 3.05 times faster than autoregressive decoding and outpaces the production EAGLE3-VL head on average and by about 11\% on grounded tasks. An entropy law explains when drafting pays, predicting the longest accepted blocks on grounded tasks, where the target's next-token entropy is lowest.\ifglancepublic\ Our code is available at \url{https://github.com/js-lee-AI/GLANCE}.\fi
\end{abstract}

\section{Introduction}
\label{sec:intro}

Vision-language models (VLMs) increasingly produce long text about an image. They describe photographs, read scanned documents, and pull values off charts~\citep{qwen3vl2025,bai2025qwen25vl}. Generation remains autoregressive, so every output token needs a full forward pass of the target model, and at small batch sizes this decode loop is bound by memory bandwidth rather than compute. The visual encoder, in contrast, runs once during prefill, and in our workloads the decode loop takes most of the end-to-end time. Faster grounded generation therefore has to come from the decode loop.

\begin{figure}[t]
\centering
\includegraphics[width=0.99\columnwidth]{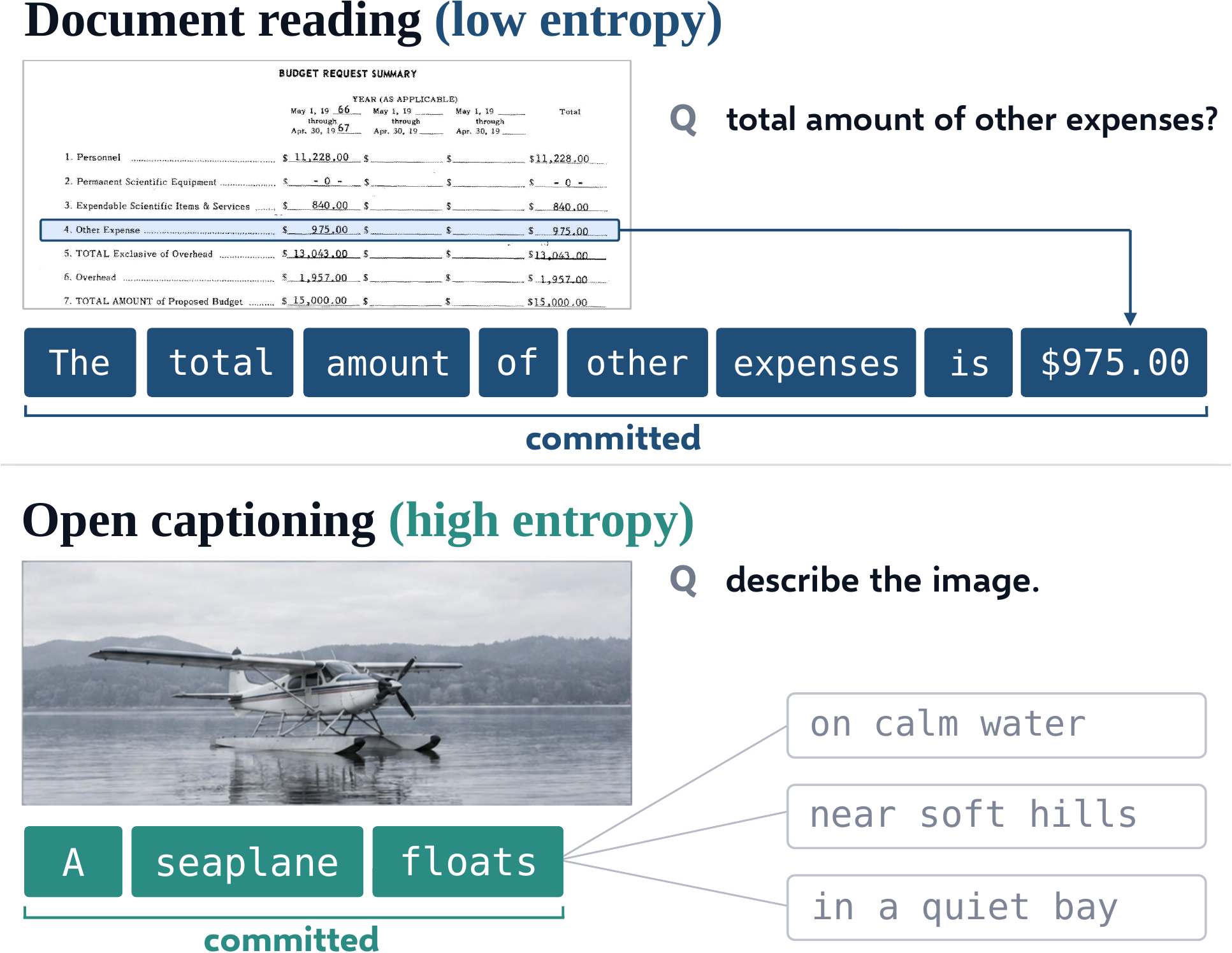}
\caption{Draftability of grounded and open generation. On a document (top), the image pins the answer and one draft pass commits the entire block. Open captioning (bottom) admits many continuations, and the same pass commits only a short prefix.}
\label{fig:teaser}
\end{figure}

Speculative decoding shortens this loop without changing the output. A cheap drafter proposes several future tokens, and the target verifies them in one forward pass, committing the longest prefix it agrees with~\citep{leviathan2023fast,chen2023accelerating}. Modern drafters are lightweight heads that read the target's own hidden features and propose a tree of candidates~\citep{cai2024medusa,li2024eagle,li2024eagle2,li2025eagle3,miao2024specinfer}, and vision-language serving stacks now ship such heads for VLM targets~\citep{aqmedai2025eagle3vl}.

Adapting this recipe to VLMs, however, has led into a cycle. The drafter is autoregressive, so a candidate $k$ tokens deep costs $k$ sequential draft passes, and the drafter is kept small so that those passes stay cheap. Each pass also attends over the whole multimodal context, in which image tokens usually outnumber text tokens, and a drafter that looks at the image pays for it again at every step. Heads that keep the full context, such as the production EAGLE3-VL head~\citep{aqmedai2025eagle3vl}, therefore keep their drafts shallow, and methods built for VLMs compress, prune, or hide the image tokens from the drafter~\citep{gagrani2024speculative,lin2025vispec,wang2025specvlm,chen2025hivis}. In both cases drafts stay short, and in the second the drafter also loses the image evidence that grounded text depends on.

On grounded workloads, however, the image makes drafting work. When a VLM answers a question about a document, much of its output already appears on the page, and the next token is often nearly deterministic and frequently copied verbatim, as Figure~\ref{fig:teaser} illustrates. We make this precise with an entropy law, under which the expected accepted length falls with the target's next-token entropy, and grounded tasks concentrate at the low-entropy end.

\begin{figure*}[t]
\centering
\includegraphics[width=\textwidth]{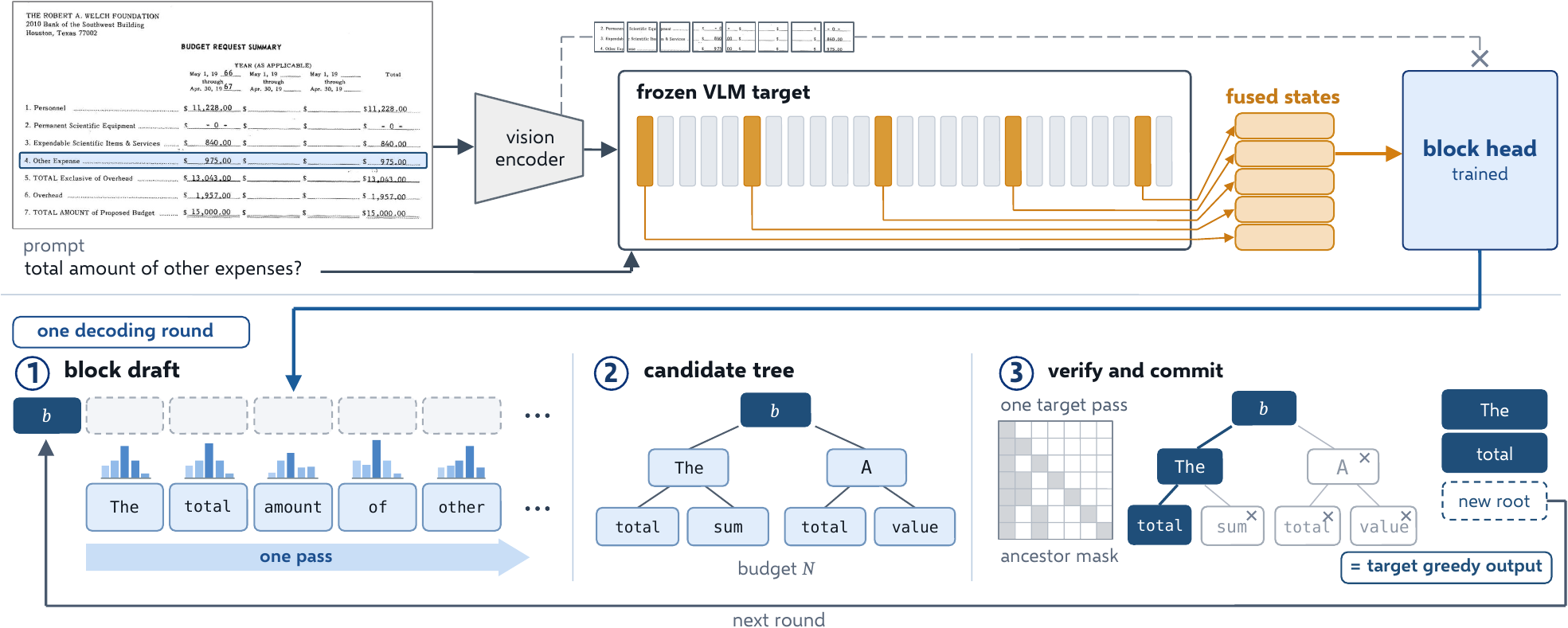}
\caption{Overview of GLANCE. The block head (top) reads the target's fused vision-language states, never raw visual tokens. In one round (bottom), one draft pass fills the block, the top prefixes form a candidate tree, and one target pass verifies it.}
\label{fig:framework}
\end{figure*}

To exploit this predictability, we propose long and wide candidate sets in a single draft pass. We present GLANCE (grounded block drafting with one-pass candidate expansion). A block-diffusion head~\citep{chen2026dflash,arriola2025blockdiffusion} reads the frozen target's fused vision-language states and, in a single forward pass, produces a distribution for every position of a block. The highest-scoring prefixes form a wide candidate tree, and the target verifies all of them in one pass. Because the head has already paid for the whole block, width costs no further draft passes~\citep{ringel2026ddtree,zhang2026caddtree,cast2026}. The committed text is exactly the target's greedy output, which we prove and also verify in fp32 on every audited prompt. One-pass block drafting has so far been demonstrated only on text-only language models, while lossless drafters built for VLMs have remained autoregressive, and GLANCE is the first one-pass block drafter that is lossless on an unmodified VLM target.

Inside one production engine at a fixed round budget, GLANCE decodes up to $3.05\times$ faster than autoregressive decoding on chart question answering and outpaces the production EAGLE3-VL head on average and by about $11\%$ on every grounded task, using one draft pass where that head uses eight. Trained on the same corpus and schedule as an EAGLE-3 head and run under the same tree, it decodes $4.5$ to $25.7\%$ faster on all five tasks, and it keeps its lead when retrained on the corpus of ViSpec, the strongest published VLM drafter. Moreover, the form of the entropy law carries over to other VLM targets and to speech, code, and chat.

This paper makes three contributions.
\begin{itemize}
\item We diagnose why speculative decoding has underdelivered on VLMs. Autoregressive drafting and reduced vision access reinforce each other, and together they forgo the long verbatim runs of grounded generation, which a one-pass drafter collects in a single pass.
\item We introduce GLANCE, the first one-pass block drafter that is lossless on an unmodified VLM target. Its head drafts a whole block in one pass from the target's fused vision-language states, a wide tree drawn from that pass is verified in one target pass, and we build it into SGLang~\citep{zheng2024sglang}, where it runs against the production head under one engine and one round budget.
\item We establish an entropy law of draftability and test it with a mismatched-image intervention and a grounded-tail analysis on five tasks, three VLM targets, and four settings beyond vision.
\end{itemize}

\section{Related Work}
\label{sec:related}

\paragraph{Draft heads and tree verification.}
Speculative decoding accepts drafted tokens with a rejection rule that provably preserves the target distribution~\citep{leviathan2023fast,chen2023accelerating}. The practical line drafts from the target's own hidden features with a lightweight head. Medusa attaches independent heads for each future offset~\citep{cai2024medusa}, Hydra makes them sequentially dependent~\citep{ankner2024hydra}, and EAGLE drafts autoregressively at the feature level over dynamic candidate trees~\citep{li2024eagle,li2024eagle2}, which EAGLE-3 extends with training-time test and multi-layer feature fusion~\citep{li2025eagle3}. Many candidates can be verified at once by packing a prefix tree into a single ancestor-masked target pass~\citep{miao2024specinfer}, and later work studies tree shape~\citep{chen2024sequoia,wang2024opttree}, theory, and benchmarking~\citep{yin2024theoretical,huang2024specdecpp,xia2024specbench}.

\paragraph{Speculative decoding for VLMs.}
The first study of speculative decoding for multimodal models found a text-only drafter to be a strong baseline~\citep{gagrani2024speculative}. Subsequent work follows two routes. One shrinks what the drafter reads, by compressing image tokens into a few adaptor embeddings~\citep{lin2025vispec}, pruning visual tokens~\citep{wang2025specvlm}, pruning video tokens under verifier guidance~\citep{ji2025specvlmvideo}, or hiding them from the drafter altogether~\citep{chen2025hivis}. The other gives a small drafter its own cheap path to vision, through multimodal distillation~\citep{liu2025massv}, dynamic draft trees~\citep{huo2025specllava}, or refined target features fused by cross-attention~\citep{wang2025dream}. DREAM uses entropy to weight that fusion inside an autoregressive drafter, whereas we use entropy to predict acceptance and draft in one pass. Semi-autoregressive multimodal drafting relaxes the one-token-a-pass constraint but gives up exactness~\citep{specflash2025}, and a benchmark now covers the setting~\citep{mmspec2026}. All of these methods keep an autoregressive drafter and engineer its access to vision, whereas GLANCE inherits vision from the target and drafts in one pass.

\paragraph{Block drafting and diffusion drafters.}
A parallel line removes the sequential bottleneck inside the drafter. It includes lookahead embeddings~\citep{monea2023pass}, Jacobi decoding~\citep{fu2024lookahead}, block-parallel discrete diffusion language models~\citep{nie2025llada,arriola2025blockdiffusion}, diffusion drafters verified by an autoregressive target~\citep{christopher2025specdiff,li2025diffuspec,cheng2025deer}, and DFlash, a block-diffusion head on the frozen target's hidden states~\citep{chen2026dflash}. DDTree and CaDDTree show that one-pass drafting makes tree width cheap for text-only targets~\citep{ringel2026ddtree,zhang2026caddtree}, and later text-only work builds on such drafters~\citep{dflare2026,wu2026dpace,zhang2026dflow,wang2026xpress,li2026dartree,kwon2026whiflash,cast2026}. Block drafting has reached VLMs only by converting the target itself into a self-speculating diffusion model~\citep{wu2026fastdvlm,zhang2026fastddrive}, which changes the model's output. GLANCE instead keeps the target frozen and its output exact.

\section{GLANCE}
\label{sec:method}

Figure~\ref{fig:framework} shows one decoding round of GLANCE. A block head drafts a whole block of future tokens in one forward pass, and the target verifies a wide tree of those candidates in one forward pass and commits its own greedy prefix. Only the head is trained, and the target stays frozen.

\subsection{Speculative Decoding and Acceptance}
\label{sec:prelim-sd}
Let $p(\cdot\mid x)$ be the frozen target's next-token distribution given the multimodal prefix $x$, which holds the text tokens and the encoded image. Greedy decoding emits $\argmax_w p(w\mid x)$, one token for each target forward pass. Speculative decoding instead proceeds in \emph{rounds}~\citep{leviathan2023fast,chen2023accelerating}. Given the committed prefix and a pending \emph{root} token $b$, a drafter proposes candidate continuations, and the target commits the longest prefix that matches its own greedy continuation $Y=(Y_1,Y_2,\dots)$ after $b$, where $Y_k=\argmax_w p(w\mid x\circ b\circ Y_{1:k-1})$. When $\acc$ nonroot tokens are accepted, the round commits $\acc+1$ tokens, namely the accepted prefix and one more token read off the last verified distribution. We report the mean acceptance length
\begin{equation}
\label{eq:mat}
\tau \;=\; \E[\acc+1]\;=\;\E[\acc]+1 ,
\end{equation}
so autoregressive decoding has $\tau{=}1$. For systems that log $\E[\acc]$ instead, we add the committed token so that all rows share one scale. A larger $\tau$ means fewer target passes for each generated token, and the wall-clock speedup is $\tau$ discounted by the cost of drafting and verification.

\subsection{One-Pass Drafting on Fused States}
\label{sec:method-drafter}
GLANCE drafts with a block-diffusion head~\citep{chen2026dflash} of $5$ layers, $1.05$B parameters, and block size $B{=}16$, attached to a frozen Qwen3-VL-8B target. The first block position holds the pending root $b$ and the other $B-1$ positions are masked, so one forward pass returns a marginal over the vocabulary at every offset $j=1,\dots,\Lmax$ with $\Lmax=B-1=15$. The head never reads raw visual tokens. Its cross-attention keys and values are the target's hidden states at five layers spread over the depth of the stack, $\{1,9,17,25,33\}$~\citep{li2025eagle3}. By the time these states exist, the target has already merged the image into its text representation, so the drafter inherits visual grounding without an encoder, an adaptor, or compressed image tokens of its own.

\begin{assumption}[One-pass marginal interface]
\label{asm:onepass}
Conditioned on the cached prefix $x$ and the pending root $b$, the drafter returns one marginal $q_j(\cdot\mid x,b)$ at every offset $j=1,\dots,\Lmax$ in a single forward pass, before any token of the block is committed. A candidate prefix is ranked by its plug-in score $\hat\pi(y_{1:\ell})=\prod_{j\le\ell}q_j(y_j\mid x,b)$.
\end{assumption}

An autoregressive drafter pays for depth with sequential passes, each attending over the full multimodal context, and the production EAGLE3-VL head takes eight such passes a round in the comparison of Section~\ref{sec:exp-sglang}. The block head pays one pass for the whole block, and the image enters that pass only through states the target has already computed.

\subsection{Wide-Tree Verification and Exactness}
\label{sec:method-tree}

\begin{definition}[Candidate tree and accept rule]
\label{def:tree}
A \emph{candidate tree} $\Tree$ is a prefix-closed set of nonempty strings after the root $b$, of depth at most $\Lmax$ and \emph{budget} $\budget=|\Tree|$. Packed into one ancestor-masked target pass, the row of a depth-$d$ node $y_{1:d}$ yields exactly $p(\cdot\mid x\circ b\circ y_{1:d})$. With $Y$ the target's greedy continuation, the \emph{accepted length} is $\Acc(\Tree)=\max\{k\ge0: Y_{1:j}\in\Tree\ \text{for all } j\le k\}$. The walk commits $Y_{1:\Acc(\Tree)}$ and emits the next target-greedy token from the last verified row as the new root.
\end{definition}

The tree builder keeps the $\budget$ prefixes with the highest plug-in score $\hat\pi$, and the kept set is prefix-closed because no prefix scores below its extensions~\citep{ringel2026ddtree,zhang2026caddtree}. The tree is packed into one ancestor-masked target pass~\citep{miao2024specinfer}, the decoder walks it along the target's greedy tokens, and Appendix~\ref{app:algo} gives the complete round as pseudocode. Growing $\budget$ adds verifier work inside that single pass but no draft passes. We use $\budget{=}63$, or a $31$-node tree that together with the root matches the production head's budget of $32$ draft tokens a round.

\begin{theorem}[Lossless greedy equivalence, informal]
\label{thm:lossless}
Assume the target's argmax is unique at every visited prefix and its top-$2$ logit gap there exceeds the numerical difference between packed and unpacked evaluation. Then, for any candidate tree, the walk of Definition~\ref{def:tree} commits exactly the target's greedy autoregressive sequence, token for token.
\end{theorem}

We call a run \emph{bitwise identical} to greedy decoding when its committed token ids equal those of the target's own greedy decoding at every position. A weak head can only shorten the accepted prefixes, and no fallback path is needed, because every committed token is the argmax of a target row. We nevertheless measure exactness instead of assuming it, and the full statement and proof are in Appendix~\ref{app:theory}. In fp32, GLANCE is bitwise identical to greedy decoding on all $60$ audited prompts. In bf16, the target does not reproduce its own greedy output across two runs even without any drafter, because kernel choice perturbs near-ties, and GLANCE reproduces that output at least as often as a second run of the target does. Exactness is therefore decided in fp32, and Appendix~\ref{app:impl} reports both audits for every system.

\subsection{Training}
\label{sec:method-training}
Only the block head is trained, while the vision tower and the target stay frozen. The training rows are the target's own greedy generations~\citep{offpolicy2026} on $8{,}000$ COCO-Caption2017~\citep{lin2014coco} and $8{,}000$ TextVQA~\citep{singh2019textvqa} prompts. The objective is top-$1$ coverage of the target's token at every block offset, with offset $j$ weighted by $0.8^{j}$, and a reveal scheme exposes a random prefix of the block so that the head learns every reveal length. The head is initialized from a released text-only block head for Qwen3-8B~\citep{chen2026dflash,yang2025qwen3} and trained for one epoch on one GPU. The recipe contains no document, infographic, or chart data. The full training card is in Appendix~\ref{app:impl}.

\section{An Entropy Law of Draftability}
\label{sec:law}

\looseness=-1 We characterize how draftable a workload is with one scalar, the target's next-token entropy $\HH$ at the round root. Entropy has been used to gate or stop drafting~\citep{agrawal2024adaedl,sage2026,mahmoud2026acceptance}, and we give its relation to acceptance a closed form that the experiments then test.

We model the matches along a block as a survival process. The accepted run continues at offset $j$ with probability $\pj$, and we approximate these hazards by a single match probability $\pmatch(\HH)$ that varies slowly with the offset. Taking the match logit to be affine in the entropy gives
\begin{equation}
\label{eq:law}
\begin{aligned}
\logit\pmatch(\HH)&=b_0-b_1\HH,\\
\Eacc\;&\approx\;\sum_{\ell=1}^{\Lmax}\pmatch^{\ell}
=\frac{\pmatch\,(1-\pmatch^{\Lmax})}{1-\pmatch},
\end{aligned}
\end{equation}
with block horizon $\Lmax{=}15$. Here $b_0$ is the zero-entropy intercept and $b_1>0$ the entropy slope, both fitted on the accepted lengths of each task's rounds with a $31$-node tree. The truncation matters only near $\HH{=}0$, since the untruncated mean $\pmatch/(1-\pmatch)$ exceeds the truncated one by the relative amount $\pmatch^{\Lmax}$, which is below $4\%$ whenever $\pmatch\le0.8$. Appendix~\ref{app:theory} derives Equation~\eqref{eq:law}, grounds its monotone part in Fano's inequality, and bounds how much round-level variance any entropy-only predictor can explain.

Equation~\eqref{eq:law} describes the top-$1$ path of a draft. A wider tree adds siblings at every depth and so raises acceptance at every entropy, and we measure this gain to be a nearly constant factor across tasks.

Two predictions follow. First, $\Eacc$ strictly decreases in $\HH$, hence any workload that systematically lowers the target's entropy yields longer accepted blocks. Grounded generation is such a workload, because it copies determinate strings such as glyphs, numbers, and answer spans from the image. Second, the law has a ceiling that grounded generation breaks.

\begin{theorem}[Grounded tail, informal]
\label{thm:tail}
Equation~\eqref{eq:law} caps the near-certain mean at $e^{b_0}$. If near-zero-entropy rounds enter, with probability $\pi$, a verbatim-copy state whose matches persist, the mean as $\HH\to0$ becomes $\pi\Lmax+(1-\pi)M$, where $M$ is the mean of the ordinary state, and it exceeds $e^{b_0}$ once $\pi$ passes an explicit threshold.
\end{theorem}

The copy regime is where one-pass drafting gains the most, because a verbatim run of length $\ell$ costs an autoregressive drafter $\ell$ sequential passes and a block drafter one. Appendix~\ref{app:theory} states the two-state model behind Theorem~\ref{thm:tail} and tests it on prompt-disjoint splits.

\begin{table*}[t]
\centering
\setlength{\tabcolsep}{4.1pt}\footnotesize
\begin{tabular}{lccccccccc}
\toprule
 & & & \multicolumn{3}{c}{EAGLE3-VL, eight draft passes} & \multicolumn{3}{c}{GLANCE, one draft pass} & \\
\cmidrule(lr){4-6}\cmidrule(lr){7-9}
Task & $\bar{\HH}$ & AR ms/tok & $\tau\ \uparrow$ & ms/tok $\downarrow$ & speedup $\uparrow$ & $\tau\ \uparrow$ & ms/tok $\downarrow$ & speedup $\uparrow$ & GLANCE faster by \\
\midrule
\grp{10}{Higher-entropy tasks}
Captioning & $0.47$ & $24.72$ & $\mathbf{3.51}$ & $\mathbf{11.54}$ & $\mathbf{2.14\times}$ & \g{$2.92$} & \g{$13.10$} & \g{$1.89\times$} & $-11.9\%$ \\
TextVQA & $0.39$ & $24.34$ & $\mathbf{4.31}$ & $\mathbf{9.56}$ & $\mathbf{2.55\times}$ & \g{$3.57$} & \g{$10.76$} & \g{$2.26\times$} & $-11.2\%$ \\
\grp{10}{Lower-entropy tasks}
InfographicVQA & $0.31$ & $24.58$ & $3.51$ & $11.99$ & $2.05\times$ & \g{$\mathbf{3.65}$} & \g{$\mathbf{10.84}$} & \g{$\mathbf{2.27\times}$} & $\mathbf{+10.6\%}$ \\
DocVQA & $0.15$ & $24.91$ & $3.68$ & $11.67$ & $2.14\times$ & \g{$\mathbf{3.91}$} & \g{$\mathbf{10.51}$} & \g{$\mathbf{2.37\times}$} & $\mathbf{+11.0\%}$ \\
ChartQA & $0.19$ & $24.18$ & $4.50$ & $8.79$ & $2.75\times$ & \g{$\mathbf{4.69}$} & \g{$\mathbf{7.93}$} & \g{$\mathbf{3.05\times}$} & $\mathbf{+10.8\%}$ \\
\midrule
Geometric mean & & & & & $2.31\times$ & & & \g{$\mathbf{2.34\times}$} & $\mathbf{+1.3\%}$ \\
\bottomrule
\end{tabular}
\caption{GLANCE against the production EAGLE3-VL head in SGLang, sharing engine, GPU, and a tree of $32$ draft tokens, filled in eight passes by EAGLE3-VL and one by GLANCE. $\bar{\HH}$ is the target's mean root entropy in nats. Blue marks GLANCE.}
\label{tab:sglang}
\end{table*}

\section{Experiments}
\label{sec:exp}

\subsection{Setup}
\label{sec:exp-setup}

\paragraph{Models and tasks.}
The target is Qwen3-VL-8B-Instruct~\citep{qwen3vl2025}, kept frozen. We evaluate five tasks ordered from open-ended to grounded, namely COCO captioning~\citep{lin2014coco}, TextVQA~\citep{singh2019textvqa}, InfographicVQA (InfoVQA)~\citep{mathew2022infographicvqa}, DocVQA~\citep{mathew2021docvqa}, and ChartQA~\citep{masry2022chartqa}. We call the last three, whose answers are read off a document or chart, the grounded tasks, and they are the three tasks of lowest mean target entropy in Table~\ref{tab:sglang}. The rest of the standard VLM suite is multiple-choice or single-word and leaves no decode loop to shorten. Captioning and TextVQA prompts for GLANCE come from the test split of each source, which is disjoint from its training rows, and the other three tasks appear in none of the primary head's training data.

\paragraph{Baselines.}
Following the EAGLE series~\citep{li2024eagle,li2024eagle2,li2025eagle3}, we compare against released methods that preserve the target's output. These are n-gram prompt lookup~\citep{saxena2023prompt}, classic two-model speculation with image-conditioned and text-only drafts, the production EAGLE3-VL head~\citep{aqmedai2025eagle3vl}, and ViSpec~\citep{lin2025vispec} together with the EAGLE-2 and Medusa baselines of its codebase. No published VLM drafter releases a head for our target, so we train the ViSpec-codebase heads ourselves with the official recipe.

\paragraph{Protocol.}
Decoding is batch one and greedy, with up to $256$ new tokens. Wall-clock speedup is the ratio of the mean decode time for a token, autoregressive over speculative, with both arms measured in one engine on one GPU. Prompt templates and counts, hardware, and results under sampling at temperature $1$ are in Appendix~\ref{app:impl}.

\begin{table*}[t]
\centering
\setlength{\tabcolsep}{4.6pt}\footnotesize
\begin{tabular}{lccccccc}
\toprule
 & & \multicolumn{5}{c}{acceptance length $\tau\ \uparrow$} & \\
\cmidrule(lr){3-7}
Method (draft passes a round) & Params & Captioning & TextVQA & InfoVQA & DocVQA & ChartQA & Lossless \\
\midrule
\grp{8}{Training-free and two-model drafting}
n-gram lookup (0) & $0$ & $1.29$ & $2.49$ & $2.53$ & $3.30$ & $2.57$ & $\checkmark$ \\
Classic SD, Qwen3-VL-4B (8) & $4.4$B & $\mathbf{3.53}$ & $\mathbf{3.79}$ & $\mathbf{3.95}$ & $\mathbf{4.39}$ & $\mathbf{4.47}$ & $\checkmark$ \\
Classic SD, Qwen3-1.7B text-only (8) & $2.0$B & $1.49$ & $1.42$ & $1.90$ & $1.68$ & $2.12$ & $\checkmark$ \\
\grp{8}{Trained draft heads}
EAGLE3-VL (5) & $0.40$B & $2.13$ & $2.66$ & $2.48$ & $2.56$ & $2.92$ & $\checkmark$ \\
EAGLE-2, ViSpec codebase (3) & $0.23$B & $2.41$ & $2.45$ & $2.38$ & $2.54$ & $2.89$ & $\checkmark^{\dagger}$ \\
ViSpec, official recipe (3) & $0.31$B & $2.45$ & $2.43$ & $2.45$ & $2.46$ & $2.95$ & $\checkmark^{\dagger}$ \\
Medusa, ViSpec codebase (1) & $0.08$B & $1.51$ & $1.52$ & $1.47$ & $1.53$ & $1.61$ & $\checkmark^{\dagger}$ \\
\rowcolor{oursrow} GLANCE (1) & $1.05$B & $\mathbf{3.09}$ & $\mathbf{3.46}$ & $\mathbf{3.75}$ & $\mathbf{3.76}$ & $\mathbf{5.12}$ & $\checkmark^{\dagger}$ \\
\grp{8}{Matched training}
EAGLE-3 head, depth-3 chain (3) & $0.40$B & $1.62$ & $1.59$ & $1.55$ & $1.58$ & $1.87$ & $\checkmark$ \\
\rowcolor{oursrow} GLANCE, budget-63 tree (1) & $1.05$B & $\mathbf{4.05}$ & $\mathbf{3.97}$ & $\mathbf{3.90}$ & $\mathbf{4.13}$ & $\mathbf{7.44}$ & $\checkmark^{\dagger}$ \\
\grp{8}{ViSpec's training corpus}
\rowcolor{oursrow} GLANCE, 1 epoch (1) & $1.05$B & $3.02$ & $3.44$ & $3.75$ & $3.78$ & $5.02$ & $\checkmark^{\dagger}$ \\
\rowcolor{oursrow} GLANCE, 21 epochs (1) & $1.05$B & $3.04$ & $3.49$ & $3.78$ & $3.85$ & $5.07$ & $\checkmark^{\dagger}$ \\
\grp{8}{ViSpec's home target Qwen2.5-VL-7B}
ViSpec, released head (3) & $0.35$B & $3.34$ & $\mathbf{3.26}$ & $3.21$ & $3.04$ & $3.59$ & $\checkmark^{\dagger}$ \\
\rowcolor{oursrow} GLANCE (1) & $1.23$B & $\mathbf{4.00}$ & $2.61$ & $\mathbf{3.47}$ & $\mathbf{3.12}$ & $\mathbf{4.72}$ & $\checkmark^{\dagger}$ \\
\bottomrule
\end{tabular}
\caption{Acceptance length on Qwen3-VL-8B except the last group, each drafter at its own operating point, GLANCE with the budget-$63$ tree. Params counts drafter weights. $\checkmark$ marks exact acceptance, $\dagger$ output audited bitwise identical to greedy decoding. Blue marks GLANCE.}
\label{tab:headline}
\end{table*}

\subsection{Head-to-Head in a Production Engine}
\label{sec:exp-sglang}

Table~\ref{tab:sglang} shows GLANCE and the production EAGLE3-VL head inside SGLang $0.5.6$ on one RTX A6000 in bf16, with both CUDA-graph captured and both verifying a tree of $32$ draft tokens a round, EAGLE3-VL its own top-$8$ dynamic tree and GLANCE its prefix tree. Engine, GPU, and round budget are thus shared, and the draft tokens come from eight sequential passes in EAGLE3-VL and from one pass in GLANCE. This eight-pass tree decodes $18$ to $33\%$ faster than EAGLE3-VL's released configuration of three passes over four draft tokens on every task. On the three lower-entropy tasks, GLANCE is the faster system by $10.6$ to $11.0\%$ and reaches $3.05\times$ the speed of autoregressive decoding on ChartQA. Every paired bootstrap interval excludes zero, GLANCE is faster on at least $81$ of the $101$ prompts of each of these tasks, and it also leads on the five-task geometric mean, by $1.3\%$ with a $95\%$ interval from $0.3$ to $2.2\%$. None of the three tasks appears in GLANCE's training data. The lead on these three tasks holds for two further training seeds and at temperature~$1$.

On captioning and TextVQA, the two tasks with the highest mean entropy $\bar{\HH}$, the eight-pass head leads. The product-of-marginals ranking of Assumption~\ref{asm:onepass} accounts for this split. It is accurate when the tokens of a block are nearly determined by the image, whereas on free-running text an autoregressive drafter stays coherent by construction. Consistently, GLANCE accepts longer blocks on every lower-entropy task than on either higher-entropy task, whereas the production head places TextVQA above both DocVQA and InfographicVQA. With an EAGLE-3 head and a GLANCE head trained on one corpus and run under the same tree, GLANCE is faster on all five tasks, by $4.5\%$ on captioning, $14.9\%$ on TextVQA, and $16.1$ to $25.7\%$ on the three grounded tasks, as Table~\ref{tab:matchedsglang} shows. Appendix~\ref{app:ledger} gives the prompt-level intervals behind Table~\ref{tab:sglang} and further wall-clock results.

\begin{figure*}[t]
\centering
\includegraphics[width=\textwidth]{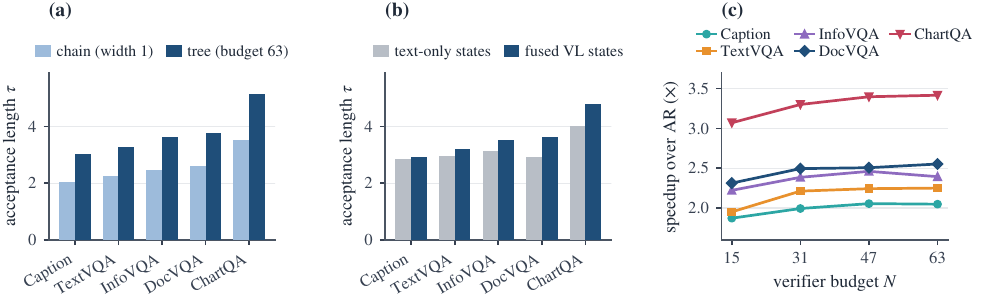}
\caption{Sources of GLANCE's acceptance. (a) The identical head as a width-1 chain and as a budget-63 tree. (b) The head reading the target's fused states or a text-only model's states, at budget 31. (c) Speedup against the verifier budget $\budget$.}
\label{fig:systems}
\end{figure*}

\subsection{Acceptance Against Released Drafters}
\label{sec:exp-acceptance}

Table~\ref{tab:headline} reports the acceptance length of every drafter released for these targets, each at its own operating point. Among trained heads, GLANCE accepts the longest blocks on all five tasks. Classic two-model speculation with a $4$B draft accepts more on four of the five tasks, but a draft half the size of the target costs about half a target pass for each drafted token, so it slows decoding on every task, as Appendix~\ref{app:classicsd} shows. The ViSpec codebase isolates what its vision adaptor adds. Trained without the adaptor, the head is exactly the EAGLE-2 drafter, and the two heads differ by at most $0.08$ in acceptance length on any task, so the target's fused states, which both heads read, already carry what the adaptor's compressed image tokens add.

Four controls separate the architecture from its training. First, trained from scratch with one $26$K-row corpus, target, global batch, schedule, and framework~\citep{specforge2026}, GLANCE decodes faster than an EAGLE-3 head on all five tasks under the shared tree of Section~\ref{sec:exp-sglang}. This corpus, unlike our primary recipe, contains document and chart rows. Second, retrained on ViSpec's own $68$K-row corpus, at one epoch and at ViSpec's twenty-one, GLANCE stays within one percent of its acceptance under our recipe, pooled over the five tasks, and its lead over the ViSpec family therefore holds on ViSpec's own data and schedule. Third, on ViSpec's home target Qwen2.5-VL-7B, a GLANCE head trained on that target accepts longer blocks than the released ViSpec head on four of the five tasks. Fourth, the released text head that GLANCE starts from, run unchanged, trails the production head on all five tasks of Table~\ref{tab:sglang}, by $15\%$ on the geometric mean, and our training lengthens its accepted blocks by $17$ to $21\%$, which turns that deficit into GLANCE's lead. Details of all four controls are in Appendix~\ref{app:impl}.

\begin{figure*}[t]
\centering
\includegraphics[width=\textwidth]{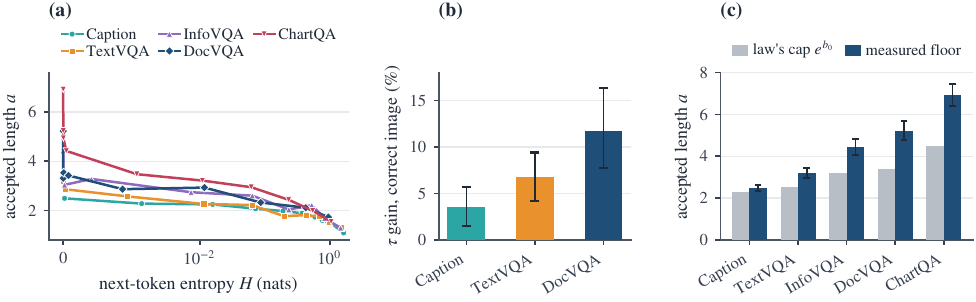}
\caption{Tests of the entropy law. (a) Accepted length by decile of the target's next-token entropy. (b) Acceptance gain from the correct over a mismatched image, with $95\%$ bootstrap intervals. (c) The fitted law's cap $e^{b_0}$ against the measured lowest-decile mean.}
\label{fig:analysis}
\end{figure*}

\subsection{Where the Acceptance Comes From}
\label{sec:exp-ablations}

\paragraph{Tree against chain.}
GLANCE's head has $2.6\times$ the parameters of EAGLE3-VL's, so a larger head might explain the gains. Panel (a) of Figure~\ref{fig:systems} measures what the tree adds by running the identical head as a width-$1$ chain. The budget-$63$ tree accepts between $1.45$ and $1.49\times$ the chain's length on every task, a nearly constant factor, as the law anticipates. The tree also decodes $1.36\times$ faster than the chain in wall-clock time, net of verifying a wider tree. Under the shared budget of Table~\ref{tab:sglang}, a GLANCE round is $4$ to $7\%$ shorter than an EAGLE3-VL round on every task, so the larger head drafts all fifteen offsets in one pass for less than the small head pays for eight.

\paragraph{Vision through fused states.}
Panel (b) replaces the target's fused states with those of a text-only Qwen3-8B, which never sees the image and whose states the head was initialized on. Acceptance falls on every task, and it falls most on grounded ones, since the text-only states retain $97\%$ of GLANCE's acceptance on captioning but only $80\%$ on DocVQA. Zeroing the conditioning states collapses $\tau$ to $1.11$, which shows that the head drafts from the target's context.

\paragraph{Width saturates early.}
Panel (c) sweeps the verifier budget. Speedup rises from $\budget{=}15$ to $31$ on every task and then flattens. Width pays the most on ChartQA, the task with the longest accepted blocks. Appendix~\ref{app:ledger} tabulates the sweep and breaks down the cost of a round.

\subsection{Testing the Entropy Law}
\label{sec:exp-law}

\paragraph{Fit across tasks.}
Panel (a) of Figure~\ref{fig:analysis} shows GLANCE's accepted length falling with entropy on all five tasks, with grounded tasks accepting more at nearly every entropy level. The first prediction of Section~\ref{sec:law} holds, since isotonic fits explain at least $96\%$ of the variance of the ten entropy-decile means on every task. The fitted slope $b_1$ steepens with grounding, from $0.46$ on captioning to $1.71$ on ChartQA. The fits are listed in Appendix~\ref{app:extra}.

\paragraph{Image and entropy.}
To test whether the image itself is responsible, we intervene on the image alone. The same prompts are decoded with the correct and with a mismatched image along a shared teacher-forced trajectory. Panel (b) shows that the correct image lengthens accepted blocks on every task, and more so as grounding increases. The correct image also lowers the target's mean entropy along that trajectory, from $0.72$ to $0.26$ nats on DocVQA, and once entropy is held fixed the remaining image effect is no longer positive. Together with the fused-state ablation, this places the image's contribution in the entropy channel.

\paragraph{The grounded tail.}
Panel (c) compares the cap $e^{b_0}$ of Theorem~\ref{thm:tail} with the measured mean of $\acc$ in each task's lowest-entropy decile. Every floor lies above its cap, and the excess grows with grounding, reaching $5.23$ against $3.41$ on DocVQA. On ChartQA, $11.5\%$ of all rounds accept eight or more tokens.

\paragraph{Transfer.}
The law keeps its form under a change of drafter, target, and modality, with a positive slope in every case. Two-model drafts on Qwen2.5-VL-7B and Qwen2-VL-7B~\citep{wang2024qwen2vl} fit at $b_1{=}0.67$ and $0.40$, and beyond vision the slope stays positive on speech recognition, chat, code, and a non-Qwen backbone. These targets all feed vision as a token prefix, and Appendix~\ref{app:extra} reports the fits and a cross-attention target.

\section{Conclusion}
\label{sec:conclusion}
We presented GLANCE, a lossless one-pass block drafter for frozen vision-language models. Its block-diffusion head reads vision through the target's fused states and fills a whole block in one pass, which suits the long verbatim runs of grounded generation. A wide tree verified in one target pass turns that block into accepted length, and the output stays exactly the target's greedy decoding. An entropy law explains when this pays and carries over across targets and modalities.

\section*{Limitations}
Our measurements are at batch one, the regime in which decoding is bound by memory bandwidth and speculative decoding is deployed for latency. We do not measure serving at large batch sizes, which changes the cost of verification. The exactness guarantee concerns greedy decoding. Under sampling, the walk draws each token from the verified row of its node and descends while the draw is a child in the tree. Every committed token is then a draw from the target's own distribution, and the output distribution is preserved, while bitwise identity no longer applies. Our characterization of draftability is stated for targets that feed vision as a token prefix, which covers the Qwen-VL family studied here. Finally, we study single-image tasks whose outputs span tens to hundreds of tokens, since tasks with single-token answers leave no decode loop to shorten.

\section*{Ethical Considerations}
This work accelerates inference of existing vision-language models without changing their greedy outputs, so it introduces no new model behavior. Risks of the underlying models, such as biased or incorrect descriptions of images, carry over unchanged and are neither amplified nor mitigated. All datasets are public research benchmarks used under their licenses, and no human subjects or personal data are involved. Faster decoding lowers the energy spent on each generated token.

\bibliography{refs,\ifglancepublic refs_concurrent_public\fi}

\begin{thebibliography}{68}
\providecommand{\natexlab}[1]{#1}

\bibitem[{Agrawal et~al.(2024)Agrawal, Jeon, and Lee}]{agrawal2024adaedl}
Sudhanshu Agrawal, Wonseok Jeon, and Mingu Lee. 2024.
\newblock \href {https://arxiv.org/abs/2410.18351} {{AdaEDL}: Early draft
  stopping for speculative decoding of large language models via an
  entropy-based lower bound on token acceptance probability}.
\newblock \emph{Preprint}, arXiv:2410.18351.

\bibitem[{Ankner et~al.(2024)Ankner, Parthasarathy, Nrusimha, Rinard,
  Ragan-Kelley, and Brandon}]{ankner2024hydra}
Zachary Ankner, Rishab Parthasarathy, Aniruddha Nrusimha, Christopher Rinard,
  Jonathan Ragan-Kelley, and William Brandon. 2024.
\newblock Hydra: Sequentially-dependent draft heads for {Medusa} decoding.
\newblock In \emph{Conference on Language Modeling (COLM)}.
\newblock ArXiv:2402.05109.

\bibitem[{{AQ-MedAI}(2026)}]{aqmedai2025eagle3vl}
{AQ-MedAI}. 2026.
\newblock {Qwen3-VL-8B-Instruct-eagle3}.
\newblock \url{https://huggingface.co/AQ-MedAI/Qwen3-VL-8B-Instruct-eagle3}.
\newblock Hugging Face model checkpoint.

\bibitem[{Arriola et~al.(2025)Arriola, Gokaslan, Chiu, Yang, Qi, Han, Sahoo,
  and Kuleshov}]{arriola2025blockdiffusion}
Marianne Arriola, Aaron Gokaslan, Justin~T. Chiu, Zhihan Yang, Zhixuan Qi,
  Jiaqi Han, Subham~Sekhar Sahoo, and Volodymyr Kuleshov. 2025.
\newblock {Block Diffusion: Interpolating Between Autoregressive and Diffusion
  Language Models}.
\newblock In \emph{International Conference on Learning Representations}.

\bibitem[{Bai et~al.(2025{\natexlab{a}})Bai, Cai, Chen, Chen, Chen, Cheng,
  Deng, Ding, Gao, Ge et~al.}]{qwen3vl2025}
Shuai Bai, Yuxuan Cai, Ruizhe Chen, Keqin Chen, Xionghui Chen, Zesen Cheng,
  Lianghao Deng, Wei Ding, Chang Gao, Chunjiang Ge, and 1 others.
  2025{\natexlab{a}}.
\newblock {Qwen3-VL Technical Report}.
\newblock \emph{arXiv preprint arXiv:2511.21631}.

\bibitem[{Bai et~al.(2025{\natexlab{b}})Bai, Chen, Liu, Wang, Ge, Song, Dang,
  Wang, Wang, Tang et~al.}]{bai2025qwen25vl}
Shuai Bai, Keqin Chen, Xuejing Liu, Jialin Wang, Wenbin Ge, Sibo Song, Kai
  Dang, Peng Wang, Shijie Wang, Jun Tang, and 1 others. 2025{\natexlab{b}}.
\newblock {Qwen2.5-VL Technical Report}.
\newblock \emph{arXiv preprint arXiv:2502.13923}.

\bibitem[{Cai et~al.(2024)Cai, Li, Geng, Peng, Lee, Chen, and
  Dao}]{cai2024medusa}
Tianle Cai, Yuhong Li, Zhengyang Geng, Hongwu Peng, Jason~D Lee, Deming Chen,
  and Tri Dao. 2024.
\newblock Medusa: Simple {LLM} inference acceleration framework with multiple
  decoding heads.
\newblock In \emph{International Conference on Machine Learning}.

\bibitem[{Chen et~al.(2023)Chen, Borgeaud, Irving, Lespiau, Sifre, and
  Jumper}]{chen2023accelerating}
Charlie Chen, Sebastian Borgeaud, Geoffrey Irving, Jean-Baptiste Lespiau,
  Laurent Sifre, and John Jumper. 2023.
\newblock Accelerating large language model decoding with speculative sampling.
\newblock \emph{arXiv preprint arXiv:2302.01318}.

\bibitem[{Chen et~al.(2024{\natexlab{a}})Chen, Chen, Zhang, Chen, Wu, Zhang,
  Chen, Li, Wan, and Wang}]{chen2024allava}
Guiming~Hardy Chen, Shunian Chen, Ruifei Zhang, Junying Chen, Xiangbo Wu, Zhiyi
  Zhang, Zhihong Chen, Jianquan Li, Xiang Wan, and Benyou Wang.
  2024{\natexlab{a}}.
\newblock \href {https://arxiv.org/abs/2402.11684} {{ALLaVA}: Harnessing
  {GPT4V}-synthesized data for lite vision-language models}.
\newblock \emph{Preprint}, arXiv:2402.11684.

\bibitem[{Chen et~al.(2026)Chen, Liang, and Liu}]{chen2026dflash}
Jian Chen, Yesheng Liang, and Zhijian Liu. 2026.
\newblock {DFlash: Block Diffusion for Flash Speculative Decoding}.
\newblock In \emph{International Conference on Machine Learning}.

\bibitem[{Chen et~al.(2024{\natexlab{b}})Chen, May, Svirschevski, Huang,
  Ryabinin, Jia, and Chen}]{chen2024sequoia}
Zhuoming Chen, Avner May, Ruslan Svirschevski, Yuhsun Huang, Max Ryabinin,
  Zhihao Jia, and Beidi Chen. 2024{\natexlab{b}}.
\newblock Sequoia: Scalable and robust speculative decoding.
\newblock In \emph{Advances in Neural Information Processing Systems}.
\newblock ArXiv:2402.12374.

\bibitem[{Cheng et~al.(2025)Cheng, Yang, Li, Deng, Guo, and Hu}]{cheng2025deer}
Zicong Cheng, Guo-Wei Yang, Jia Li, Zhijie Deng, Meng-Hao Guo, and Shi-Min Hu.
  2025.
\newblock \href {https://arxiv.org/abs/2512.15176} {{DEER}: Draft with
  diffusion, verify with autoregressive models}.
\newblock \emph{Preprint}, arXiv:2512.15176.

\bibitem[{Christopher et~al.(2025)Christopher, Bartoldson, Ben-Nun, Cardei,
  Kailkhura, and Fioretto}]{christopher2025specdiff}
Jacob~K Christopher, Brian~R Bartoldson, Tal Ben-Nun, Michael Cardei, Bhavya
  Kailkhura, and Ferdinando Fioretto. 2025.
\newblock Speculative diffusion decoding: Accelerating language generation
  through diffusion.
\newblock In \emph{Proceedings of the North American Chapter of the Association
  for Computational Linguistics (NAACL)}.
\newblock ArXiv:2408.05636.

\bibitem[{Fu et~al.(2024)Fu, Bailis, Stoica, and Zhang}]{fu2024lookahead}
Yichao Fu, Peter Bailis, Ion Stoica, and Hao Zhang. 2024.
\newblock Break the sequential dependency of {LLM} inference using lookahead
  decoding.
\newblock In \emph{Proceedings of the International Conference on Machine
  Learning (ICML)}.
\newblock ArXiv:2402.02057.

\bibitem[{Gagrani et~al.(2024)Gagrani, Goel, Jeon, Park, Lee, and
  Lott}]{gagrani2024speculative}
Mukul Gagrani, Raghavv Goel, Wonseok Jeon, Junyoung Park, Mingu Lee, and
  Christopher Lott. 2024.
\newblock On speculative decoding for multimodal large language models.
\newblock \emph{arXiv preprint arXiv:2404.08856}.
\newblock ELVM @ CVPR 2024.

\bibitem[{Ganesan et~al.(2025)Ganesan, Segal, Aggarwal, Sinnadurai, Lie, and
  Thangarasa}]{liu2025massv}
Mugilan Ganesan, Shane Segal, Ankur Aggarwal, Nish Sinnadurai, Sean Lie, and
  Vithursan Thangarasa. 2025.
\newblock {MASSV}: Multimodal adaptation and self-data distillation for
  speculative decoding of vision-language models.
\newblock In \emph{Findings of EMNLP}.

\bibitem[{Hu et~al.(2025)Hu, Xia, Liu, Raman, Liu, Bao, Sather, Thangarasa, and
  Zhang}]{wang2025dream}
Yunhai Hu, Tianhua Xia, Zining Liu, Rahul Raman, Xingyu Liu, Bo~Bao, Eric
  Sather, Vithursan Thangarasa, and Sai~Qian Zhang. 2025.
\newblock {DREAM}: Drafting with refined target features and entropy-adaptive
  cross-attention fusion for multimodal speculative decoding.
\newblock In \emph{Advances in Neural Information Processing Systems}.
\newblock ArXiv:2505.19201.

\bibitem[{Huang et~al.(2025{\natexlab{a}})Huang, Yang, Liu, Yin, Li, Ren, and
  Barsoum}]{wang2025specvlm}
Haiduo Huang, Fuwei Yang, Zhenhua Liu, Xuanwu Yin, Dong Li, Pengju Ren, and
  Emad Barsoum. 2025{\natexlab{a}}.
\newblock \href {https://arxiv.org/abs/2509.11815} {{SpecVLM}: Fast speculative
  decoding in vision-language models}.
\newblock \emph{Preprint}, arXiv:2509.11815.

\bibitem[{Huang et~al.(2025{\natexlab{b}})Huang, Guo, and
  Wang}]{huang2024specdecpp}
Kaixuan Huang, Xudong Guo, and Mengdi Wang. 2025{\natexlab{b}}.
\newblock {SpecDec++}: Boosting speculative decoding via adaptive candidate
  lengths.
\newblock In \emph{Conference on Language Modeling (COLM)}.

\bibitem[{Hui et~al.(2024)Hui, Yang, Cui, Yang, Liu, Zhang, Liu, Zhang, Yu, Lu
  et~al.}]{hui2024qwen25coder}
Binyuan Hui, Jian Yang, Zeyu Cui, Jiaxi Yang, Dayiheng Liu, Lei Zhang, Tianyu
  Liu, Jiajun Zhang, Bowen Yu, Keming Lu, and 1 others. 2024.
\newblock {Qwen2.5-Coder Technical Report}.
\newblock \emph{arXiv preprint arXiv:2409.12186}.

\bibitem[{Huo et~al.(2025)Huo, Zhang, Wang, Xu, Chen, Tai, and
  Chen}]{huo2025specllava}
Mingxiao Huo, Jiayi Zhang, Hewei Wang, Jinfeng Xu, Zheyu Chen, Huilin Tai, and
  Yijun Chen. 2025.
\newblock {Spec-LLaVA}: Accelerating vision-language models with dynamic
  tree-based speculative decoding.
\newblock \emph{arXiv preprint arXiv:2509.11961}.

\bibitem[{Ji et~al.(2025)Ji, Zhang, Xia, Chen, Shou, Chen, and
  Li}]{ji2025specvlmvideo}
Yicheng Ji, Jun Zhang, Heming Xia, Jinpeng Chen, Lidan Shou, Gang Chen, and
  Huan Li. 2025.
\newblock {SpecVLM}: Enhancing speculative decoding of video llms via
  verifier-guided token pruning.
\newblock In \emph{Proceedings of the Conference on Empirical Methods in
  Natural Language Processing (EMNLP)}.
\newblock ArXiv:2508.16201.

\bibitem[{Kang et~al.(2025)Kang, Shu, Li, Zhai, and Chen}]{lin2025vispec}
Jialiang Kang, Han Shu, Wenshuo Li, Yingjie Zhai, and Xinghao Chen. 2025.
\newblock {ViSpec}: Accelerating vision-language models with vision-aware
  speculative decoding.
\newblock In \emph{Advances in Neural Information Processing Systems}.

\bibitem[{Kwon et~al.(2023)Kwon, Li, Zhuang, Sheng, Zheng, Yu, Gonzalez, Zhang,
  and Stoica}]{kwon2023vllm}
Woosuk Kwon, Zhuohan Li, Siyuan Zhuang, Ying Sheng, Lianmin Zheng, Cody~Hao Yu,
  Joseph Gonzalez, Hao Zhang, and Ion Stoica. 2023.
\newblock Efficient memory management for large language model serving with
  {PagedAttention}.
\newblock \emph{SOSP}.

\bibitem[{Kwon et~al.(2026)Kwon, Williams, Li, Kouris, and
  Venieris}]{kwon2026whiflash}
Young~D. Kwon, Miles Williams, Rui Li, Alexandros Kouris, and Stylianos~I.
  Venieris. 2026.
\newblock {WhiFlash}: Accelerating speculative decoding with token-level
  cross-paradigm routing.
\newblock In \emph{Proceedings of the Conference on Empirical Methods in
  Natural Language Processing (EMNLP)}.
\newblock ArXiv:2606.07710.

\bibitem[{Lee and Eo(2026)}]{cast2026}
Jungseob Lee and Sugyeong Eo. 2026.
\newblock \href {https://arxiv.org/abs/2610.00321} {{CAST: Cost-Aware
  Speculative Trees from One-Pass Block Drafters}}.
\newblock \emph{Preprint}, arXiv:2610.00321.

\bibitem[{Lee et~al.(2026)Lee, Park, Eo, and Moon}]{offpolicy2026}
Jungseob Lee, Chanjun Park, Sugyeong Eo, and Hyeonseok Moon. 2026.
\newblock \href {https://arxiv.org/abs/2609.38795} {{Recovering Off-Policy
  Supervision for Speculative Decoding}}.
\newblock \emph{Preprint}, arXiv:2609.38795.

\bibitem[{Leviathan et~al.(2023)Leviathan, Kalman, and
  Matias}]{leviathan2023fast}
Yaniv Leviathan, Matan Kalman, and Yossi Matias. 2023.
\newblock Fast inference from transformers via speculative decoding.
\newblock In \emph{International Conference on Machine Learning}.

\bibitem[{Li et~al.(2026{\natexlab{a}})Li, Fu, Fang, Zhao, Tang, Yuan, and
  Wang}]{li2025diffuspec}
Guanghao Li, Zhihui Fu, Min Fang, Qibin Zhao, Ming Tang, Chun Yuan, and Jun
  Wang. 2026{\natexlab{a}}.
\newblock {DiffuSpec: Unlocking Diffusion Language Models for Speculative
  Decoding}.
\newblock In \emph{Findings of ACL}.

\bibitem[{Li et~al.(2026{\natexlab{b}})Li, Wang, Zhu, Wang, Yin, Shi, Chen,
  Dong, Chen, Pan et~al.}]{specforge2026}
Shenggui Li, Chao Wang, Yikai Zhu, Yubo Wang, Fan Yin, Shuai Shi, Yefei Chen,
  Xiaomin Dong, Qiaoling Chen, Jin Pan, and 1 others. 2026{\natexlab{b}}.
\newblock {SpecForge}: A flexible and efficient open-source training framework
  for speculative decoding.
\newblock \emph{arXiv preprint arXiv:2603.18567}.

\bibitem[{Li et~al.(2026{\natexlab{c}})Li, Luo, Shang, and
  Shen}]{li2026dartree}
Tianyi Li, Yaxin Luo, Xinyi Shang, and Zhiqiang Shen. 2026{\natexlab{c}}.
\newblock \href {https://arxiv.org/abs/2608.13524} {{DARTree}: Speculative
  diffusion decoding with autoregressive draft trees}.
\newblock \emph{Preprint}, arXiv:2608.13524.

\bibitem[{Li et~al.(2024{\natexlab{a}})Li, Wei, Zhang, and
  Zhang}]{li2024eagle2}
Yuhui Li, Fangyun Wei, Chao Zhang, and Hongyang Zhang. 2024{\natexlab{a}}.
\newblock {EAGLE}-2: Faster inference of language models with dynamic draft
  trees.
\newblock In \emph{Empirical Methods in Natural Language Processing}.

\bibitem[{Li et~al.(2024{\natexlab{b}})Li, Wei, Zhang, and Zhang}]{li2024eagle}
Yuhui Li, Fangyun Wei, Chao Zhang, and Hongyang Zhang. 2024{\natexlab{b}}.
\newblock {EAGLE}: Speculative sampling requires rethinking feature
  uncertainty.
\newblock In \emph{International Conference on Machine Learning}.

\bibitem[{Li et~al.(2025)Li, Wei, Zhang, and Zhang}]{li2025eagle3}
Yuhui Li, Fangyun Wei, Chao Zhang, and Hongyang Zhang. 2025.
\newblock {EAGLE}-3: Scaling up inference acceleration of large language models
  via training-time test.
\newblock In \emph{Advances in Neural Information Processing Systems}.
\newblock ArXiv:2503.01840.

\bibitem[{Lin et~al.(2014)Lin, Maire, Belongie, Hays, Perona, Ramanan,
  Doll{\'a}r, and Zitnick}]{lin2014coco}
Tsung-Yi Lin, Michael Maire, Serge Belongie, James Hays, Pietro Perona, Deva
  Ramanan, Piotr Doll{\'a}r, and C~Lawrence Zitnick. 2014.
\newblock Microsoft {COCO}: Common objects in context.
\newblock In \emph{ECCV}.

\bibitem[{Liu et~al.(2024)Liu, Li, Li, and Lee}]{liu2024llava15}
Haotian Liu, Chunyuan Li, Yuheng Li, and Yong~Jae Lee. 2024.
\newblock \href {https://doi.org/10.1109/CVPR52733.2024.02484} {Improved
  baselines with visual instruction tuning}.
\newblock In \emph{Proceedings of the IEEE/CVF Conference on Computer Vision
  and Pattern Recognition (CVPR)}.

\bibitem[{Mahmoud(2026)}]{mahmoud2026acceptance}
Saif Mahmoud. 2026.
\newblock \href {https://arxiv.org/abs/2604.14682} {Acceptance dynamics across
  cognitive domains in speculative decoding}.
\newblock \emph{Preprint}, arXiv:2604.14682.

\bibitem[{Masry et~al.(2022)Masry, Long, Tan, Joty, and
  Hoque}]{masry2022chartqa}
Ahmed Masry, Do~Xuan Long, Jia~Qing Tan, Shafiq Joty, and Enamul Hoque. 2022.
\newblock {ChartQA}: A benchmark for question answering about charts with
  visual and logical reasoning.
\newblock In \emph{Findings of ACL}.

\bibitem[{Mathew et~al.(2022)Mathew, Bagal, Tito, Karatzas, Valveny, and
  Jawahar}]{mathew2022infographicvqa}
Minesh Mathew, Viraj Bagal, Rub{\`e}n Tito, Dimosthenis Karatzas, Ernest
  Valveny, and CV~Jawahar. 2022.
\newblock {InfographicVQA}.
\newblock In \emph{WACV}.

\bibitem[{Mathew et~al.(2021)Mathew, Karatzas, and Jawahar}]{mathew2021docvqa}
Minesh Mathew, Dimosthenis Karatzas, and CV~Jawahar. 2021.
\newblock {DocVQA}: A dataset for {VQA} on document images.
\newblock In \emph{WACV}.

\bibitem[{{Meta}(2024)}]{meta2024llama32vision}
{Meta}. 2024.
\newblock \href
  {https://huggingface.co/meta-llama/Llama-3.2-11B-Vision-Instruct}
  {{Llama-3.2-11B-Vision-Instruct}}.
\newblock
  \url{https://huggingface.co/meta-llama/Llama-3.2-11B-Vision-Instruct}.
\newblock Hugging Face model card.

\bibitem[{Miao et~al.(2024)Miao, Oliaro, Zhang, Cheng, Wang, Zhang, Wong, Zhu,
  Yang, Shi, Shi, Chen, Arfeen, Abhyankar, and Jia}]{miao2024specinfer}
Xupeng Miao, Gabriele Oliaro, Zhihao Zhang, Xinhao Cheng, Zeyu Wang, Zhengxin
  Zhang, Rae Ying~Yee Wong, Alan Zhu, Lijie Yang, Xiaoxiang Shi, Chunan Shi,
  Zhuoming Chen, Daiyaan Arfeen, Reyna Abhyankar, and Zhihao Jia. 2024.
\newblock {SpecInfer}: Accelerating large language model serving with
  tree-based speculative inference and verification.
\newblock In \emph{ASPLOS}.

\bibitem[{Monea et~al.(2023)Monea, Joulin, and Grave}]{monea2023pass}
Giovanni Monea, Armand Joulin, and Edouard Grave. 2023.
\newblock {PaSS}: Parallel speculative sampling.
\newblock \emph{arXiv preprint arXiv:2311.13581}.
\newblock NeurIPS 2023 Workshop on Efficient Natural Language and Speech
  Processing.

\bibitem[{Nie et~al.(2025)Nie, Zhu, You, Zhang, Ou, Hu, Zhou, Lin, Wen, and
  Li}]{nie2025llada}
Shen Nie, Fengqi Zhu, Zebin You, Xiaolu Zhang, Jingyang Ou, Jun Hu, Jun Zhou,
  Yankai Lin, Ji-Rong Wen, and Chongxuan Li. 2025.
\newblock {Large Language Diffusion Models}.
\newblock In \emph{Advances in Neural Information Processing Systems}.

\bibitem[{Ringel and Romano(2026)}]{ringel2026ddtree}
Liran Ringel and Yaniv Romano. 2026.
\newblock {Accelerating Speculative Decoding with Block Diffusion Draft Trees}.
\newblock In \emph{Conference on Language Modeling (COLM)}.

\bibitem[{Saxena(2023)}]{saxena2023prompt}
Apoorv Saxena. 2023.
\newblock \href {https://github.com/apoorvumang/prompt-lookup-decoding} {Prompt
  lookup decoding}.
\newblock \url{https://github.com/apoorvumang/prompt-lookup-decoding}.
\newblock GitHub repository.

\bibitem[{Shen et~al.(2026)Shen, Wang, Zhang, Hsieh, Han, Wan, Zhang, Zhang,
  Xiong, Liu, Zhang, Cao, Zhao, and Zhang}]{mmspec2026}
Hui Shen, Xin Wang, Ping Zhang, Yunta Hsieh, Qi~Han, Zhongwei Wan, Ziheng
  Zhang, Jingxuan Zhang, Jing Xiong, Ziyuan Liu, Yifan Zhang, Hangrui Cao,
  Chenyang Zhao, and Mi~Zhang. 2026.
\newblock \href {https://arxiv.org/abs/2603.14989} {{MMSpec}: Benchmarking
  speculative decoding for vision-language models}.
\newblock \emph{Preprint}, arXiv:2603.14989.

\bibitem[{Singh et~al.(2019)Singh, Natarajan, Shah, Jiang, Chen, Batra, Parikh,
  and Rohrbach}]{singh2019textvqa}
Amanpreet Singh, Vivek Natarajan, Meet Shah, Yu~Jiang, Xinlei Chen, Dhruv
  Batra, Devi Parikh, and Marcus Rohrbach. 2019.
\newblock Towards {VQA} models that can read.
\newblock In \emph{CVPR}.

\bibitem[{{Team OLMo} et~al.(2024){Team OLMo}, Walsh, Soldaini, Groeneveld, Lo,
  Arora, Bhagia, Gu, Huang, Jordan et~al.}]{olmo2024olmo2}
{Team OLMo}, Pete Walsh, Luca Soldaini, Dirk Groeneveld, Kyle Lo, Shane Arora,
  Akshita Bhagia, Yuling Gu, Shengyi Huang, Matt Jordan, and 1 others. 2024.
\newblock {2 OLMo 2 Furious}.
\newblock \emph{arXiv preprint arXiv:2501.00656}.

\bibitem[{Tong et~al.(2026)Tong, Zhang, Wan, Lin, Yuan, and Hu}]{sage2026}
Yujia Tong, Tian Zhang, Yunyang Wan, Kaiwei Lin, Jingling Yuan, and Chuang Hu.
  2026.
\newblock \href {https://arxiv.org/abs/2602.00523} {{SAGE}: Accelerating
  vision-language models via entropy-guided adaptive speculative decoding}.
\newblock \emph{Preprint}, arXiv:2602.00523.

\bibitem[{Wang et~al.(2025{\natexlab{a}})Wang, Su, Li, Xia, Ye, Duan, Wang, and
  Zhang}]{wang2024opttree}
Jikai Wang, Yi~Su, Juntao Li, Qingrong Xia, Zi~Ye, Xinyu Duan, Zhefeng Wang,
  and Min Zhang. 2025{\natexlab{a}}.
\newblock {OPT-Tree}: Speculative decoding with adaptive draft tree structure.
\newblock \emph{Transactions of the Association for Computational Linguistics
  (TACL)}.
\newblock ArXiv:2406.17276.

\bibitem[{Wang et~al.(2024)Wang, Bai, Tan, Wang, Fan, Bai, Chen, Liu, Wang, Ge
  et~al.}]{wang2024qwen2vl}
Peng Wang, Shuai Bai, Sinan Tan, Shijie Wang, Zhihao Fan, Jinze Bai, Keqin
  Chen, Xuejing Liu, Jialin Wang, Wenbin Ge, and 1 others. 2024.
\newblock {Qwen2-VL}: Enhancing vision-language model's perception of the world
  at any resolution.
\newblock \emph{arXiv preprint arXiv:2409.12191}.

\bibitem[{Wang et~al.(2026)Wang, Wertheimer, Lim, Srivatsa, Ganti, Zhang, and
  Wang}]{wang2026xpress}
Zheng Wang, Davis Wertheimer, Yu~Chin~Fabian Lim, Mudhakar Srivatsa, Raghu~K.
  Ganti, Minjia Zhang, and Naigang Wang. 2026.
\newblock \href {https://arxiv.org/abs/2608.02438} {{xPress}: Parallel
  refinement for diffusion drafters in speculative decoding}.
\newblock \emph{Preprint}, arXiv:2608.02438.

\bibitem[{Wang et~al.(2025{\natexlab{b}})Wang, Li, Du, Zhou, Zhang, and
  Yang}]{specflash2025}
Zihua Wang, Ruibo Li, Haozhe Du, Joey~Tianyi Zhou, Yu~Zhang, and Xu~Yang.
  2025{\natexlab{b}}.
\newblock \href {https://arxiv.org/abs/2505.12728} {{SpecFLASH}: A
  latent-guided semi-autoregressive speculative decoding framework for
  efficient multimodal generation}.
\newblock \emph{Preprint}, arXiv:2505.12728.

\bibitem[{Wolf et~al.(2020)Wolf, Debut, Sanh, Chaumond, Delangue, Moi, Cistac,
  Rault, Louf, Funtowicz, Davison, Shleifer, von Platen, Ma, Jernite, Plu, Xu,
  Le~Scao, Gugger, Drame, Lhoest, and Rush}]{wolf2020transformers}
Thomas Wolf, Lysandre Debut, Victor Sanh, Julien Chaumond, Clement Delangue,
  Anthony Moi, Pierric Cistac, Tim Rault, Remi Louf, Morgan Funtowicz, Joe
  Davison, Sam Shleifer, Patrick von Platen, Clara Ma, Yacine Jernite, Julien
  Plu, Canwen Xu, Teven Le~Scao, Sylvain Gugger, and 3 others. 2020.
\newblock \href {https://doi.org/10.18653/v1/2020.emnlp-demos.6} {Transformers:
  State-of-the-art natural language processing}.
\newblock In \emph{Proceedings of the 2020 Conference on Empirical Methods in
  Natural Language Processing: System Demonstrations}.

\bibitem[{Wu et~al.(2026{\natexlab{a}})Wu, Lan, Fu, Gao, Wang, Yu, Alvarez,
  Molchanov, Luo, Han, Zhu, and Xie}]{wu2026fastdvlm}
Chengyue Wu, Shiyi Lan, Yonggan Fu, Sensen Gao, Jin Wang, Jincheng Yu, Jose~M.
  Alvarez, Pavlo Molchanov, Ping Luo, Song Han, Ligeng Zhu, and Enze Xie.
  2026{\natexlab{a}}.
\newblock \href {https://arxiv.org/abs/2604.06832} {{Fast-dVLM}: Efficient
  block-diffusion vlm via direct conversion from autoregressive vlm}.
\newblock \emph{arXiv preprint arXiv:2604.06832}.

\bibitem[{Wu et~al.(2026{\natexlab{b}})Wu, Yao, Qi, Zheng, Wang, Ma, Liao,
  Lakkaraju, Li, and Du}]{wu2026dpace}
Tianyu Wu, Yu~Yao, Zhenting Qi, Han Zheng, Zhuohan Wang, Haoran Ma, Lawrence
  Liao, Himabindu Lakkaraju, Ju~Li, and Yilun Du. 2026{\natexlab{b}}.
\newblock \href {https://arxiv.org/abs/2605.18810} {{D-PACE}: Dynamic
  position-aware cross-entropy for parallel speculative drafting}.
\newblock \emph{Preprint}, arXiv:2605.18810.

\bibitem[{Xia et~al.(2024)Xia, Yang, Dong, Wang, Li, Ge, Liu, Li, and
  Sui}]{xia2024specbench}
Heming Xia, Zhe Yang, Qingxiu Dong, Peiyi Wang, Yongqi Li, Tao Ge, Tianyu Liu,
  Wenjie Li, and Zhifang Sui. 2024.
\newblock Unlocking efficiency in large language model inference: A
  comprehensive survey of speculative decoding.
\newblock In \emph{Findings of the Association for Computational Linguistics
  (ACL Findings)}.
\newblock Spec-Bench.

\bibitem[{Xie et~al.(2026)Xie, Wang, Qiu, and Cheng}]{chen2025hivis}
Zhinan Xie, Peisong Wang, Shuang Qiu, and Jian Cheng. 2026.
\newblock {HiViS}: Hiding visual tokens from the drafter for speculative
  decoding in vision-language models.
\newblock In \emph{CVPR Findings}.

\bibitem[{Xu et~al.(2025)Xu, Guo, He, Hu, He, Bai, Chen, Wang, Fan, Dang
  et~al.}]{xu2025qwen25omni}
Jin Xu, Zhifang Guo, Jinzheng He, Hangrui Hu, Ting He, Shuai Bai, Keqin Chen,
  Jialin Wang, Yang Fan, Kai Dang, and 1 others. 2025.
\newblock {Qwen2.5-Omni Technical Report}.
\newblock \emph{arXiv preprint arXiv:2503.20215}.

\bibitem[{Yang et~al.(2025)Yang, Li, Yang, Zhang, Hui, Zheng, Yu, Gao, Huang,
  Lv et~al.}]{yang2025qwen3}
An~Yang, Anfeng Li, Baosong Yang, Beichen Zhang, Binyuan Hui, Bo~Zheng, Bowen
  Yu, Chang Gao, Chengen Huang, Chenxu Lv, and 1 others. 2025.
\newblock {Qwen3 Technical Report}.
\newblock \emph{arXiv preprint arXiv:2505.09388}.

\bibitem[{Yin et~al.(2024)Yin, Chen, Huang, and Wang}]{yin2024theoretical}
Ming Yin, Minshuo Chen, Kaixuan Huang, and Mengdi Wang. 2024.
\newblock A theoretical perspective for speculative decoding algorithm.
\newblock In \emph{Advances in Neural Information Processing Systems
  (NeurIPS)}.
\newblock ArXiv:2411.00841.

\bibitem[{Zhang et~al.(2026{\natexlab{a}})Zhang, Yu, Liu, Yu, Li, Zhu, Duo,
  Xiong, Song, Yu et~al.}]{dflare2026}
Jiebin Zhang, Zhenghan Yu, Song Liu, Eugene~J. Yu, Zheng Li, Dawei Zhu,
  Jiangshan Duo, Weimin Xiong, Yifan Song, Guanghua Yu, and 1 others.
  2026{\natexlab{a}}.
\newblock \href {https://arxiv.org/abs/2606.02091} {{DFlare}: Scaling up draft
  capacity for block diffusion speculative decoding}.
\newblock \emph{Preprint}, arXiv:2606.02091.

\bibitem[{Zhang et~al.(2026{\natexlab{b}})Zhang, Wang, Gao, Wu, Cao, Han,
  Ivanovic, Liu, Pavone, Han, Zhou, and Xie}]{zhang2026fastddrive}
Kewei Zhang, Jin Wang, Sensen Gao, Chengyue Wu, Yulong Cao, Songyang Han, Boris
  Ivanovic, Langechuan Liu, Marco Pavone, Song Han, Daquan Zhou, and Enze Xie.
  2026{\natexlab{b}}.
\newblock \href {https://arxiv.org/abs/2605.23163} {{Fast-dDrive}: Efficient
  block-diffusion {VLM} for autonomous driving}.
\newblock \emph{Preprint}, arXiv:2605.23163.

\bibitem[{Zhang et~al.(2026{\natexlab{c}})Zhang, Qiu, He, and
  Dai}]{zhang2026caddtree}
Shuai Zhang, Huachuan Qiu, Hongliang He, and Yong Dai. 2026{\natexlab{c}}.
\newblock \href {https://arxiv.org/abs/2606.01813} {{Cost-Aware Diffusion Draft
  Trees for Speculative Decoding}}.
\newblock \emph{Preprint}, arXiv:2606.01813.

\bibitem[{Zhang et~al.(2026{\natexlab{d}})Zhang, Zhang, Cui, and
  Miao}]{zhang2026dflow}
Yaojie Zhang, Linfeng Zhang, Bin Cui, and Xupeng Miao. 2026{\natexlab{d}}.
\newblock \href {https://arxiv.org/abs/2609.06498} {{DFlow}: Enabling verifier
  information flow in block diffusion speculative decoding}.
\newblock \emph{Preprint}, arXiv:2609.06498.

\bibitem[{Zhao et~al.(2024)Zhao, Ren, Hessel, Cardie, Choi, and
  Deng}]{zhao2024wildchat}
Wenting Zhao, Xiang Ren, Jack Hessel, Claire Cardie, Yejin Choi, and Yuntian
  Deng. 2024.
\newblock \href {https://openreview.net/forum?id=qGl2GUM3Xt} {{WildChat}: 1m
  {ChatGPT} interaction logs in the wild}.
\newblock In \emph{International Conference on Learning Representations
  (ICLR)}.

\bibitem[{Zheng et~al.(2024)Zheng, Yin, Xie, Sun, Huang, Yu, Cao, Kozyrakis,
  Stoica, Gonzalez, Barrett, and Sheng}]{zheng2024sglang}
Lianmin Zheng, Liangsheng Yin, Zhiqiang Xie, Chuyue Sun, Jeff Huang, Cody Yu,
  Shiyi Cao, Christos Kozyrakis, Ion Stoica, Joseph Gonzalez, Clark Barrett,
  and Ying Sheng. 2024.
\newblock \href {https://doi.org/10.52202/079017-2000} {{SGLang}: Efficient
  execution of structured language model programs}.
\newblock In \emph{Advances in Neural Information Processing Systems
  (NeurIPS)}.

\end{thebibliography}

\appendix
\raggedbottom

\section{Theoretical Statements and Proofs}
\label{app:theory}

\subsection{Losslessness}
\label{app:theory-lossless}

\begin{thmlosslessfull}
For the decoder of Definition~\ref{def:tree}, assume (float-exactness) that at every visited prefix $x\circ b\circ y_{1:k}$ the target $\argmax$ is unique and the top logit gap $\gamma_k:=\ell_{(1)}-\ell_{(2)}>0$ exceeds the packed-versus-unpacked logit perturbation. Then for any candidate tree the committed token sequence equals the target's autoregressive greedy sequence exactly, and losslessness can fail only at a margin $\gamma_k$ below that perturbation, that is, at a near-tie.
\end{thmlosslessfull}

\begin{proof}
\emph{Single round.} By Definition~\ref{def:tree} the verifier row at any prefix $x\circ b\circ y_{1:k}$ reproduces $p(\cdot\mid x\circ b\circ y_{1:k})$ exactly, since ancestor-only masking and token packing change the attention layout, not the conditioning of a row. Let $Y_{k+1}$ be the target greedy token there. If $y_{1:k}\circ Y_{k+1}\in\Tree$, the walk descends and commits $Y_{k+1}$, and otherwise it stops and emits $Y_{k+1}$ as the next root. Either way the token produced at that position is $Y_{k+1}$, and the tree affects only how many tokens the round commits.

\emph{Induction over rounds.} The emitted stream concatenates accepted paths and corrections, and each round's correction is the next round's root. Every produced token is the target greedy token at its own prefix, so induction on the output position gives the committed sequence $Y_1Y_2\cdots$ bitwise, independent of all budget choices.

\emph{Tie sensitivity.} The only step that can break is the equality of the packed and the autoregressive $\argmax$. It fails only when $\gamma_k$ falls below the packed-versus-unpacked logit perturbation, that is, at a near-tie. In fp32 the perturbation is below every observed margin, and the fp32 audit shows zero divergences.
\end{proof}

\subsection{Derivation of the Law}
\label{app:theory-law}

\begin{assumption}[Offset survival with marginal hazards]
\label{asm:survival}
Conditioned on the round, or equivalently on $\HH$, let the drafter's top-$1$ token match the target greedy token at offset $j$ with probability $\pj$, and let $\acc$ be the leading run of matches before the first miss, capped at $\Lmax$. We assume the survival factorizes into the marginal hazards,
\begin{equation}
\label{eq:survival}
\begin{aligned}
\Prob[\acc\ge\ell\mid\HH]&=\prod_{j\le\ell}\pj,\\
\Eacc&=\sum_{\ell=1}^{\Lmax}\prod_{j\le\ell}\pj .
\end{aligned}
\end{equation}
This is an assumption, not a consequence of the chain rule, since it replaces the conditional hazards $\Prob[\mathrm{match}_j\mid\mathrm{match}_{<j},\HH]$ by the marginals $\pj$.
\end{assumption}

\begin{lemma}[Slowly varying hazard gives a truncated-geometric mean]
\label{lem:trunc}
Under Assumption~\ref{asm:survival} and the further approximation $\pj\approx\pmatch$ for all $j\le\Lmax$, Equation~\eqref{eq:survival} reduces to $\Eacc\approx\sum_{\ell=1}^{\Lmax}\pmatch^{\ell} =\pmatch(1-\pmatch^{\Lmax})/(1-\pmatch)$, which tends to $\pmatch/(1-\pmatch)$ as $\Lmax\to\infty$. The untruncated form is an upper bound for every finite $\Lmax$, with relative truncation error $\pmatch^{\Lmax}$.
\end{lemma}

\begin{proof}
Setting $\pj\equiv\pmatch$ in Equation~\eqref{eq:survival} gives the finite geometric series. Its truncated tail is
\[
\sum_{\ell>\Lmax}\pmatch^{\ell}=\pmatch^{\Lmax}\cdot\frac{\pmatch}{1-\pmatch},
\]
so the relative truncation error is $\pmatch^{\Lmax}$. At $\Lmax{=}15$,
the error $\pmatch^{\Lmax}$ is $0.002$, $0.035$, and $0.206$ at $\pmatch=0.65$, $0.80$, and $0.90$, so the closed form is tight except near the entropy floor.

Two errors are folded in, the truncation tail and the slowly varying step itself. The step overstates $\Eacc$ when $\pj$ decays in $j$. It understates the mean at extremely low entropy when one $\pmatch$ is fitted across a range that contains a sharp spike at low $\HH$. The supremum of the DocVQA fit is $3.34$ at $b_0{=}1.226$ and $\Lmax{=}15$, below the measured DocVQA bottom-decile mean of $5.23$, which Theorem~\ref{thm:tail} explains.

Substituting $\pmatch(\HH)=\sigm(b_0-b_1\HH)$ gives Equation~\eqref{eq:law}, which we use as an operational characterization of the conditional mean, not as an exact survival identity.
\end{proof}

\begin{proposition}[Entropy controls the match, and the affine logit is a surrogate]
\label{prop:link}
Let $\pmatch$ be the probability that the drafter's top-$1$ token at the root equals the target greedy token, and $\HH$ the target's root entropy. (i)~The target's own top-$1$ error $e^\star=1-\max_w p(w\mid x,b)$ obeys Fano's inequality
\[
\HH\le\Hop_b(e^\star)+e^\star\log(|\Sigma|-1),
\]
so larger $\HH$ forces larger $e^\star$, and $\HH\to0$ forces $\max_w p\to1$. (ii)~Modeling the drafter's logit error as logistic with scale $s$ gives $\pmatch=\sigm(\Delta/s)$ for the target's top-$2$ margin $\Delta(\HH)$, which decreases in $\HH$ by~(i). (iii)~A first-order expansion of $\Delta$ in $\HH$ gives $\logit\pmatch\approx b_0-b_1\HH$ with $b_1=-\Delta'(\HH_0)/s>0$. Part~(i) predicts sign and monotonicity, and parts~(ii) and~(iii) are modeling steps that select the functional form.
\end{proposition}

\begin{proof}
\begin{enumerate}
\item[(i)] View the target's top-$1$ token as an estimator of a draw $W\sim p$. Fano's inequality bounds $\HH$ by an increasing function of the error $e^\star$ on $[0,1-1/|\Sigma|]$, so $e^\star$ grows with $\HH$ and vanishes as $\HH\to0$.
\item[(ii)] The drafter preserves the target's ordering if and only if the perturbed margin stays positive. A difference of logistics is logistic, so
\[
\logit\pmatch=\Delta/s,
\]
and on a shell of fixed entropy $\Delta$ decreases in $\HH$ by~(i).
\item[(iii)] This is the first-order expansion. Its adequacy over the empirical window is what the curve $R^2$ measures.
\end{enumerate}
A fuller derivation under a temperature family, in which the target's logits are a fixed shape scaled by an inverse temperature, shows that $\logit p^{\mathrm{tgt}}_{(1)}(\HH)$ is affine to $R^2=0.84$ to $0.90$ over the resolvable window $\HH\in[0.05,1.2]$ for Zipf, geometric-gap, and near-binary logit shapes, with self-confidence slopes $\kappa\approx3.2$ to $3.9$.

This reduces the match slope to a single drafter-noise scale $s$ through $b_1=\kappa/s$. Matching the measured offset-$1$ slopes gives $s=3.4$, $7.7$, and $3.8$ on captioning, TextVQA, and DocVQA.
The affine form is exactly a local linearization, and since the binary closed form is provably curved globally, the law is stated over the empirical window only.
\end{proof}

\begin{corgroundedfull}
Under the untruncated form of Equation~\eqref{eq:law},
\[
\begin{aligned}
\partial_\HH\pmatch&=-b_1\pmatch(1-\pmatch),\\
\partial_\HH\Eacc&=\frac{\partial_\HH\pmatch}{(1-\pmatch)^2}\\
&=-b_1\Eacc<0,
\end{aligned}
\]
so $\Eacc$ is strictly decreasing in $\HH$, and the truncated form is strictly increasing in $\pmatch$ and hence decreasing in $\HH$ as well. Any regime with systematically lower next-token entropy therefore has a longer expected accepted block. Grounded and OCR generation copies determinate glyph and answer strings from the image and thereby collapses the target's entropy. Grounded spans are therefore more draftable than open text, and the mean accepted length of a task is ordered by grounding.
\end{corgroundedfull}

\begin{proof}
The derivative computation is displayed in the statement and uses only $b_1>0$ and $\pmatch\in(0,1)$. The ordering claim uses only the measured monotone decrease of $\acc$ in $\HH$, not the fitted closed form. The measured round logs order the mean accepted length by grounding, at $1.88$ on captioning, $2.12$ on TextVQA, and $3.08$ on DocVQA, and the InfoVQA and ChartQA logs extend the ordering to $2.67$ and $3.71$, the latter under ChartQA's templated prompts. The gap between DocVQA and captioning, $+1.20$, has a $95\%$ bootstrap interval of $[1.07,1.33]$ and a Mann-Whitney $p=8.7\times10^{-68}$, and the grounded task carries both the higher $b_0$ and the steeper $b_1$, as Table~\ref{tab:lawfit} lists.
\end{proof}

\subsection{The Noise Ceiling}
\label{app:theory-ceiling}

\begin{proposition}[Entropy-only noise ceiling]
\label{prop:ceiling}
Among all $\sigma(\HH)$-measurable predictors $g(\HH)$ of $\acc$, the largest attainable coefficient of determination is $\bar R^2=\Var(\E[\acc\mid\HH])/\Var(\acc)=1-\E[\Var(\acc\mid\HH)]/\Var(\acc)$, attained by the conditional mean. For a truncated-geometric $\acc$ the within-round variance is of order $\E[\acc\mid\HH]^2$, so $\bar R^2$ is intrinsically small regardless of the predictor's quality.
\end{proposition}

\begin{proof}
For any $g$, orthogonality of the conditional mean gives
\[
\begin{aligned}
\E[(\acc-g(\HH))^2]&=\E[\Var(\acc\mid\HH)]\\
&\quad+\E[(\E[\acc\mid\HH]-g(\HH))^2],
\end{aligned}
\]
which is minimized at $g=\E[\acc\mid\HH]$, and the law of total variance gives the stated ratio. For the untruncated geometric, with $\mu=\E[\acc\mid\HH]$,
\[
\Var(\acc\mid\HH)=\mu(1+\mu)\ge\mu^2,
\]
so the within-round variance dominates unless $\Var(\mu(\HH))$ is comparably large. A nonparametric estimate of the between-to-total variance over $50$ bins gives ceilings of $0.110$, $0.144$, and $0.188$ on captioning, TextVQA, and DocVQA. The law's round-level $R^2$ values of $0.097$, $0.091$, and $0.074$ therefore realize $89\%$, $63\%$, and $39\%$ of what any entropy-only predictor could reach, and Table~\ref{tab:lawfit} lists the same fractions for InfoVQA and ChartQA. This is why the curve $R^2$ is the appropriate score.
\end{proof}

\subsection{The Grounded Tail}
\label{app:theory-tail}

\begin{assumption}[Two-state determinate or uncertain survival]
\label{asm:twostate}
Conditioned on the round, the matches along a block follow a two-state hidden Markov chain over $s_j\in\{D,U\}$, with entry distribution $(\pi,1-\pi)$, match probability $\rho_s$ in state $s$, transition matrix $T$ upon a match, the run ending at the first miss, and a cap at $\Lmax$. State $D$ models verbatim copying ($\rho_D\to1$, sticky $t_{DD}>0$), and $U$ is the ordinary regime with a moderate hazard. The parameters governing $\HH\to0$ are estimated in regime, from near-zero-entropy rounds only, rather than extrapolated from the affine fit. The geometric law of Lemma~\ref{lem:trunc} is the degenerate case $\rho_D=\rho_U$.
\end{assumption}

\begin{thmtailfull}
Let $M(p)=\sum_{\ell=1}^{\Lmax}p^{\ell}$. (A)~Under the single-hazard law of Equation~\eqref{eq:law}, $\sup_{\HH\ge0}\Eacc=M(\sigm(b_0))\le e^{b_0}$, a rigorous cap that needs no fit. (B)~Under Assumption~\ref{asm:twostate}, with $\mathbf v=(\pi,1-\pi)$, $D_{\boldsymbol\rho}=\operatorname{diag}(\rho_D,\rho_U)$, and $K=TD_{\boldsymbol\rho}$,
\begin{equation}
\label{eq:hmm}
\begin{aligned}
\Prob[\acc\ge\ell\mid\HH]&=\mathbf v D_{\boldsymbol\rho}K^{\ell-1}\mathbf 1,\\
\Eacc&=\mathbf v D_{\boldsymbol\rho}\Big(\textstyle\sum_{m=0}^{\Lmax-1}K^{m}\Big)\mathbf 1,
\end{aligned}
\end{equation}
whose effective hazard $h_\ell$ is not constant. Taking $\rho_D,t_{DD}\to1$ gives $\Eacc^{\HH\to0}\to\pi\Lmax+(1-\pi)M(\rho_U)$, which exceeds $e^{b_0}$ exactly when $\pi>\pi^\star:=(e^{b_0}-M(\rho_U))/(\Lmax-M(\rho_U))$. This threshold lies in $(0,1)$ because $M(\rho_U)<e^{b_0}<\Lmax$, where on DocVQA $\Lmax{=}15$ and $e^{b_0}{=}3.41$. (C)~With in-regime maximum-likelihood parameters, Equation~\eqref{eq:hmm} recovers the measured DocVQA floor mean of $5.23$, against the geometric cap of $3.41$. Proposition~\ref{prop:falsify} gives the falsifiable validation.
\end{thmtailfull}

\begin{proof}
\begin{enumerate}
\item[(A)] $M$ is strictly increasing on $(0,1)$, and $\pmatch(\HH)$ is maximized as $\HH\to0$, where it equals $\sigm(b_0)$. Then $M(p)\le p/(1-p)$ gives the bound $e^{b_0}$.
\item[(B)] Starting from $\mathbf v$, the joint event of a match at offset $1$ followed by state $s$ has mass $(\mathbf v D_{\boldsymbol\rho})_s$, and each accepted offset multiplies by $K=TD_{\boldsymbol\rho}$, so
\[
\Prob[\acc\ge\ell\mid\HH]=\mathbf v D_{\boldsymbol\rho}K^{\ell-1}\mathbf 1,
\]
and the tail-sum identity gives the mean. Survival is a nonnegative combination of the two eigenmodes of $K$, so the hazard $h_\ell$ moves monotonically toward $\lambda_{\max}(K)$ and is not constant unless $\rho_D=\rho_U$. Its first value $\pi\rho_D+(1-\pi)\rho_U$ can approach $1$ and exceed $\sigm(b_0)$, which the constant-hazard family cannot do. As $\rho_D,t_{DD}\to1$,
\[
\Prob[\acc\ge\ell]\to\pi\quad\text{for all }\ell\le\Lmax,
\]
which gives the stated excess over the cap.
\item[(C)] This is the numerical instantiation, which recovers the measured floor mean against the geometric cap.\qedhere
\end{enumerate}
\end{proof}

\begin{proposition}[Prompt-disjoint held-out falsification]
\label{prop:falsify}
Take the $122$ bottom-entropy-decile DocVQA rounds, from $47$ prompts, and split them by prompt into disjoint sets. (i)~On the full floor the geometric run length is rejected, with Pearson $\chi^2=105.3$, $\mathrm{df}=7$, $p=8.7\times10^{-20}$, while the likelihood ratio decisively favors the two-state form, with $2\Delta\ell=86.5$, $\mathrm{df}=4$, $p=7.2\times10^{-18}$. (ii)~Frozen on one prompt set and scored on the disjoint one, the two-state model attains a smaller held-out $\chi^2$ than the geometric in $6$ of $7$ tests, over both split directions and five cross-validation folds, with a mean held-out $\chi^2$ of $4.4$ against $11.9$.
\end{proposition}

\begin{proof}
The proposition records the outcome of the stated estimation procedure, with the construction and statistics as listed.
\end{proof}

\section{Decoding Round}
\label{app:algo}

Algorithm~\ref{alg:round} gives one decoding round of GLANCE. The draft pass and the verify pass are the only two forward passes of a round, and the tree budget $\budget$ enters only the size of the verify pass.

\begin{algorithm}[t]
\caption{One decoding round of GLANCE}
\label{alg:round}
\small
\begin{algorithmic}[1]
\Require cached prefix $x$ with the target's fused states, pending root $b$, budget $\budget$, horizon $\Lmax$
\State $(q_1,\dots,q_\Lmax)\gets\textsc{BlockHead}(x,b)$ \Comment{one draft pass}
\State $\Tree\gets$ the $\budget$ prefixes $y_{1:\ell}$ of highest $\prod_{j\le\ell}q_j(y_j)$
\State pack $b$ and $\Tree$ with an ancestor mask
\State rows $\gets\textsc{Target}(x,\text{packed tree})$ \Comment{one verify pass}
\State $v\gets b$;\ $\text{out}\gets[\,]$
\Loop
  \State $t\gets\argmax_w \text{rows}[v](w)$ \Comment{target greedy token}
  \If{child $v\circ t\in\Tree$}
    \State append $t$ to out;\ $v\gets v\circ t$
  \Else
    \State \Return out, new root $t$
  \EndIf
\EndLoop
\end{algorithmic}
\end{algorithm}

\section{Implementation Details}
\label{app:impl}

\paragraph{Primary drafter.}
The primary drafter is a block-diffusion head. It reads the target's fused hidden states and proposes the whole block in one forward pass, with the vision tower and the target decoder frozen. We use the checkpoint at the end of its one-epoch run, and Table~\ref{tab:traincard} lists its configuration. The remaining constants are an intermediate size of $12288$, a block horizon $\Lmax{=}15$, an image long side of $768$ pixels, AdamW betas of $0.9$ and $0.95$ with no weight decay, gradient clipping at $1.0$, seed $0$, one A6000-class GPU, and bf16 weights of $2.1$\,GB.

\begin{table}[t]
\centering
\small
\fitcol{%
\begin{tabular}{ll}
\toprule
Setting & Value \\
\midrule
\grp{2}{Architecture}
layers & $5$ \\
parameters & $1.05$B \\
block size $B$ & $16$ \\
hidden size & $4096$ \\
attention heads & $32\times128$ \\
kept target layers & $\{1,9,17,25,33\}$ \\
\grp{2}{Objective}
loss & coverage CE, offset weight $0.8^{j}$ \\
reveal & variable-mask block prefix \\
\grp{2}{Data}
source & target greedy, $256$ new tokens \\
prompts & $8000$ COCO $+$ $8000$ TextVQA \\
rows & $16$K teacher-forced \\
\grp{2}{Optimization}
optimizer, lr & AdamW, $4\times10^{-5}$ \\
epochs, accumulation & $1$, $8$ \\
max sequence length & $2048$ \\
warm start & Qwen3-8B DFlash head \\
\bottomrule
\end{tabular}}
\caption{Training card of the primary Qwen3-VL-8B drafter.}
\label{tab:traincard}
\end{table}

\paragraph{Evaluation prompts.}
The drafter trains on COCO-Caption2017 \texttt{val} rows $100$ to $8099$ and on TextVQA \texttt{validation} rows $0$ to $8005$. The home-target head trains on the first $8000$ rows of the same two splits. Every GLANCE row of the two main tables therefore takes its captioning and TextVQA prompts from the \texttt{test} split of each source, with no overlap with any training row, namely $100$ prompts in each task on Qwen3-VL-8B and $200$ captioning and $120$ TextVQA prompts on the home target, the counts of ViSpec's own evaluation. In Table~\ref{tab:sglang}, all three systems decode these same \texttt{test} prompts. The other baselines of Table~\ref{tab:headline} draw their prompts from the \texttt{validation} rows on Qwen3-VL-8B, and the released ViSpec head runs its native loaders on its home target. InfographicVQA, DocVQA, and ChartQA appear in no training corpus of the primary head. The analyses of Figures~\ref{fig:systems} and~\ref{fig:analysis} compare conditions of one head on shared prompts and draw their captioning and TextVQA prompts from the \texttt{validation} rows. Every system of the two main tables on Qwen3-VL-8B, apart from the ViSpec-codebase heads with their official loaders, uses the same prompt form for each task. Captioning prompts ask the target to describe the image in detail, and TextVQA prompts are the dataset questions as written. The DocVQA prompt reads ``Read the document image and answer the question. Quote the exact text from the document, then briefly explain.'' before the question, and the InfographicVQA and ChartQA prompts ask in the same way for the exact text, numbers, values, or labels the target reads, followed by an explanation, so that these tasks also leave a decode loop to shorten.

\paragraph{Second and third targets.}
The Qwen2.5-VL-7B head is trained from scratch with the OCR-augmented recipe on that target's own greedy generations, with $8000$ caption, $8000$ TextVQA, and $10{,}180$ OCR rows from DocVQA, ChartQA, and InfoVQA prompts disjoint from the evaluation prompts, $26{,}180$ rows in total, for seven epochs at global batch $48$. It is evaluated with the same budget-$63$ tree and fp32 audit as the primary target, with captioning and TextVQA prompts from the \texttt{test} splits. Prompt counts match the ViSpec row task by task, $200$ for captioning and DocVQA and $120$ for TextVQA, which is where ViSpec's staged image set ends, and the fp32 audit finds zero divergences at eight prompts in each task. The Llama-3.2-11B-Vision~\citep{meta2024llama32vision} adapter is warm-started from the released Llama-3.1-8B DFlash head, whose held-out top-$1$ accuracy at the first offset climbs from $0.71$ to $0.78$.

\paragraph{Matched-training comparison.}
Both architectures are trained from scratch under one protocol, with the same $26{,}196$-row OCR-augmented corpus, the same frozen target, global batch $24$, three epochs, and SpecForge, and with each method's default learning rate scaled by the square root of the batch. Each drafter then runs at its standard operating point, ours with a budget-$63$ tree and the EAGLE-3 head as a depth-$3$ chain, greedy, with $64$ new tokens, in one Hugging Face implementation~\citep{wolf2020transformers} on $256$ held-out prompts. Pooled over the five tasks the acceptance ratio is $2.73$, at $4.37$ against $1.60$, and a depth-$7$ chain adds $0.01$ to the EAGLE-3 head. The ALLaVA replication retrains both heads on an ALLaVA-Instruct corpus~\citep{chen2024allava} under the same protocol and gives $2.04$ against $1.29$ pooled over the five tasks. Timed in the same implementation on one A100, with $20$ prompts in each task and $128$ new tokens, GLANCE decodes at $2.36$ to $2.59\times$ autoregressive decoding against $1.13$ to $1.17\times$ for the EAGLE-3 head.

The same two heads also run in SGLang under the protocol of Table~\ref{tab:sglang}, with the tree of $32$ draft tokens, $101$ prompts in each task after a warm-up prompt, and $256$ new tokens. Each task's two arms decode back to back on one RTX A6000. The corpus holds DocVQA and InfographicVQA \texttt{validation} rows and ChartQA \texttt{test} rows, so these three tasks take their prompts from the DocVQA and InfographicVQA \texttt{test} splits and the ChartQA \texttt{validation} split under the corpus's prompt template, and captioning and TextVQA use the \texttt{test} prompts of Table~\ref{tab:sglang}. Table~\ref{tab:matchedsglang} reports the result. GLANCE is faster on $73$ to $100$ of the $101$ prompts of each task.

\begin{table}[t]
\centering
\setlength{\tabcolsep}{3pt}\footnotesize
\fitcol{%
\begin{tabular}{lccccc}
\toprule
 & \multicolumn{2}{c}{EAGLE-3 head} & \multicolumn{2}{c}{GLANCE} & \\
\cmidrule(lr){2-3}\cmidrule(lr){4-5}
Task & $\tau\ \uparrow$ & ms/tok $\downarrow$ & $\tau\ \uparrow$ & ms/tok $\downarrow$ & GLANCE faster by \\
\midrule
Captioning & $3.29$ & $12.41$ & \g{$3.22$} & \g{$\mathbf{11.87}$} & $\mathbf{+4.5\%}$ \small{$[3.3, 5.8]$} \\
TextVQA & $2.85$ & $14.36$ & \g{$\mathbf{3.11}$} & \g{$\mathbf{12.50}$} & $\mathbf{+14.9\%}$ \small{$[12.7, 17.2]$} \\
InfographicVQA & $2.78$ & $15.06$ & \g{$\mathbf{3.06}$} & \g{$\mathbf{12.96}$} & $\mathbf{+16.1\%}$ \small{$[14.0, 18.3]$} \\
DocVQA & $3.00$ & $15.19$ & \g{$\mathbf{3.45}$} & \g{$\mathbf{12.40}$} & $\mathbf{+22.5\%}$ \small{$[19.0, 26.8]$} \\
ChartQA & $3.49$ & $11.43$ & \g{$\mathbf{4.13}$} & \g{$\mathbf{9.10}$} & $\mathbf{+25.7\%}$ \small{$[23.4, 28.0]$} \\
\bottomrule
\end{tabular}}
\caption{Matched-training heads in SGLang under the shared tree of $32$ draft tokens. Brackets give paired bootstrap $95\%$ intervals, and blue marks GLANCE.}
\label{tab:matchedsglang}
\end{table}

\paragraph{The released text head without our training.}
GLANCE starts from the released block head for Qwen3-8B~\citep{chen2026dflash}, which has the same configuration and weight shapes. Run unchanged on Qwen3-VL-8B under the protocol of Table~\ref{tab:sglang}, with EAGLE3-VL decoding the same prompts back to back on one RTX A6000, it is slower than EAGLE3-VL on every task, and every paired bootstrap interval lies below zero, as Table~\ref{tab:untrained} shows. EAGLE3-VL is the faster arm on at least $66$ of the $101$ prompts of each task. Our training lengthens the head's accepted blocks by $17$ to $21\%$.

\begin{table}[t]
\centering
\setlength{\tabcolsep}{3pt}\footnotesize
\fitcol{%
\begin{tabular}{lcccc}
\toprule
 & \multicolumn{3}{c}{$\tau\ \uparrow$} & \\
\cmidrule(lr){2-4}
Task & EAGLE3-VL & Released & GLANCE & Released faster by \\
\midrule
Captioning & $3.51$ & $2.40$ & \g{$2.92$} & $-26.4\%$ \small{$[-27.4, -25.3]$} \\
TextVQA & $4.31$ & $2.96$ & \g{$3.57$} & $-26.2\%$ \small{$[-27.9, -24.4]$} \\
InfographicVQA & $3.51$ & $3.07$ & \g{$3.65$} & $-7.3\%$ \small{$[-9.2, -5.4]$} \\
DocVQA & $3.68$ & $3.34$ & \g{$3.91$} & $-4.4\%$ \small{$[-7.0, -1.8]$} \\
ChartQA & $4.50$ & $3.91$ & \g{$4.69$} & $-7.8\%$ \small{$[-9.6, -5.8]$} \\
\midrule
Geometric mean & & & & $-15.0\%$ \small{$[-15.8, -14.1]$} \\
\bottomrule
\end{tabular}}
\caption{The released text head for Qwen3-8B, run unchanged on Qwen3-VL-8B in SGLang under the shared tree of $32$ draft tokens, against EAGLE3-VL. Margins follow Table~\ref{tab:sglang}, with paired bootstrap $95\%$ intervals. The GLANCE column repeats Table~\ref{tab:sglang}. Blue marks GLANCE.}
\label{tab:untrained}
\end{table}

\paragraph{ViSpec's corpus and training length.}
The training rows are the $68$K LLaVA-Pretrain~\citep{liu2024llava15} questions and images regenerated under ViSpec's own data pipeline, with the same source file, shuffle seed, prompt template, pixel bounds, and sampling temperature, and with responses produced by the frozen target, since self-generation is part of our recipe. The head then trains with the published card unchanged, once for one epoch and once for twenty-one, the latter continuing from the one-epoch checkpoint at the same constant learning rate and global batch of eight rows. Our own corpus receives the same treatment. All arms are evaluated with the budget-$63$ tree, branching $8$, $100$ prompts in each task, $256$ new tokens, and the fp32 audit at eight prompts in each task, back to back on one otherwise idle GPU. Figure~\ref{fig:epochcurve} evaluates every checkpoint of both runs. Acceptance saturates within about five epochs on either corpus, and twenty-one epochs add $3.4\%$ on our corpus and $0.9\%$ on ViSpec's.

\begin{figure}[t]
\centering
\includegraphics[width=\columnwidth]{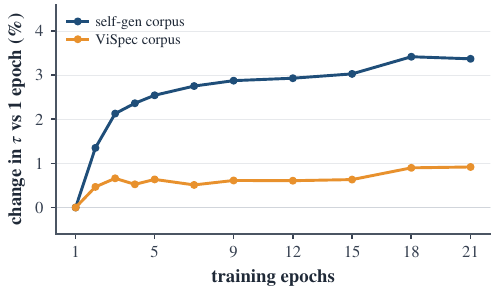}
\caption{Change in pooled acceptance against the one-epoch head, for every checkpoint of the two $21$-epoch runs, read with the budget-$63$ tree at $50$ prompts in each task.}
\label{fig:epochcurve}
\end{figure}

\paragraph{ViSpec-codebase heads on our target.}
These heads are trained on Qwen3-VL-8B-Instruct with ViSpec's official codebase and two-stage schedule, namely text-only drafter pre-training on the authors' ShareGPT split to convergence followed by multimodal fine-tuning on target-generated data, and Medusa is trained through the same codebase. They decode with the codebase's tree search and acceptance rule under its official tree of $30$ tokens at depth $3$ with top-$8$ branching. We port the codebase to Qwen3-VL-8B so that the target runs its reference forward pass, with multimodal rotary positions and deepstack visual features. Each head decodes the $100$ prompts of each task that the GLANCE row uses, with the same $896$-pixel images, greedy, batch one, and up to $256$ new tokens. The EAGLE-2 row is this recipe without the vision adaptor, and the ViSpec row is the same schedule with the adaptor, both trained from the same stage-one checkpoint, data, and seed, so the gap between the rows is the adaptor's contribution alone.

\paragraph{ViSpec on its home target.}
The released ViSpec head for Qwen2.5-VL-7B-Instruct runs with its native loaders and its default tree of $30$ tokens at depth $3$, with top-$8$ branching and two query heads, greedy, batch one, and up to $1024$ new tokens. ViSpec logs the bonus-excluded length, which we raise by the committed token.

\paragraph{Other protocol details.}
Tables~\ref{tab:sglang} and~\ref{tab:headline} run the same GLANCE head, the first in SGLang with images at native resolution and a tree of $32$ draft tokens, the second in our Hugging Face implementation at $896$ pixels with the budget-$63$ tree, so their acceptance lengths differ. In Table~\ref{tab:sglang} the three arms of each task decode back to back on one RTX A6000, and GLANCE's draft pass replays from a CUDA graph over a static key-value cache, with one graph for each number of tokens committed in the previous round. The mean entropy $\bar{\HH}$ of Table~\ref{tab:sglang} is the target's mean root entropy over the rounds of the law fits in Section~\ref{sec:exp-law}. Acceptance length counts tokens rather than time, so the drafters of Table~\ref{tab:headline} need not share one implementation. EAGLE3-VL, the production AQ-MedAI checkpoint for this target, and the n-gram baseline run in vLLM $0.22.1$~\citep{kwon2023vllm}, with $20$ prompts in each task and the warm-up sample excluded, and EAGLE3-VL runs as a chain of $5$, which Table~\ref{tab:gammasweep} sweeps from $3$ to $15$. Prompt lookup proposes a draft only when an n-gram match exists, so its acceptance length is the length of the blocks it does propose. The classic two-model rows are the endpoints of the three-draft sweep of Table~\ref{tab:classicsd}. The mismatched-image probe uses $60$ prompts in each task, with the three image conditions scored on a shared teacher-forced trajectory and prompt-clustered bootstrap intervals. Speedup is the ratio of mean decode time for a token over prompts.

\paragraph{Output equivalence.}
Our implementation decodes $20$ prompts in each of three tasks, and each system's generated tokens are compared position by position with the target's own greedy decoding of the same prompts, in the same arithmetic. Table~\ref{tab:equiv} reports the result. Decoding the target twice with no drafter shows that bf16 identity measures kernel arithmetic rather than any drafter, so the guarantee is decided in fp32, where both the drafter-free control and GLANCE reach $60$ of $60$. The ViSpec-codebase heads and the released ViSpec head, compared the same way on $21$ prompts in each task with up to $512$ new tokens, reproduce all $63$ in fp32, as exact acceptance requires.

\begin{table}[t]
\centering
\setlength{\tabcolsep}{4pt}\footnotesize
\fitcol{%
\begin{tabular}{llc}
\toprule
System & Arithmetic & Identical \\
\midrule
\grp{3}{Our implementation}
autoregressive, no drafter & bf16 & $25/60$ \\
autoregressive, no drafter & fp32 & $60/60$ \\
\rowcolor{oursrow} GLANCE, budget-$63$ tree & bf16 & $27/60$ \\
\rowcolor{oursrow} GLANCE, budget-$63$ tree & fp32 & $\mathbf{60/60}$ \\
\rowcolor{oursrow} GLANCE, matched-training head & bf16 & $33/60$ \\
EAGLE-3 head, matched training & bf16 & $19/60$ \\
\grp{3}{ViSpec codebase}
EAGLE-2, ViSpec codebase & fp32 & $63/63$ \\
ViSpec, official recipe & fp32 & $63/63$ \\
Medusa, ViSpec codebase & fp32 & $63/63$ \\
ViSpec, released head, home target & fp32 & $63/63$ \\
\bottomrule
\end{tabular}}
\caption{Output equivalence with the target's greedy decoding, with both arms in the stated arithmetic. Blue marks GLANCE.}
\label{tab:equiv}
\end{table}

\paragraph{Sampling.}
At temperature $1$, with a top-$p$ of $1$, no top-$k$ cut, and the protocol of Table~\ref{tab:sglang} otherwise unchanged, GLANCE is faster than EAGLE3-VL on the three lower-entropy tasks, by $7.5\%$ on InfographicVQA, $14.9\%$ on DocVQA, and $9.9\%$ on ChartQA. Every paired bootstrap interval lies above zero, GLANCE is the faster arm on at least $73$ of the $101$ prompts of each of these tasks, and the geometric-mean margin over the five tasks is $+2.0\%$, with a $95\%$ interval from $0.8$ to $3.2\%$. Both arms verify with SGLang's tree sampling, which commits draws from the target's own distribution.

\begin{table}[t]
\centering
\setlength{\tabcolsep}{4pt}\footnotesize
\fitcol{%
\begin{tabular}{lcccccc}
\toprule
 & \multicolumn{3}{c}{acceptance length $\tau$} & \multicolumn{3}{c}{ms/token} \\
\cmidrule(lr){2-4}\cmidrule(lr){5-7}
Chain length & Cap. & TVQA & Doc & Cap. & TVQA & Doc \\
\midrule
$3$ & $\mathbf{2.17}$ & $2.47$ & $2.36$ & $6.77$ & $6.78$ & $7.46$ \\
$5$ & $2.10$ & $\mathbf{2.65}$ & $\mathbf{2.56}$ & $7.54$ & $6.93$ & $7.62$ \\
$7$ & $1.97$ & $2.43$ & $2.50$ & $8.63$ & $8.42$ & $8.16$ \\
$10$ & $1.88$ & $2.24$ & $2.37$ & $9.93$ & $9.32$ & $9.19$ \\
$15$ & $1.80$ & $2.23$ & $2.33$ & $11.96$ & $10.68$ & $10.68$ \\
\bottomrule
\end{tabular}}
\caption{EAGLE3-VL chain-length sweep in vLLM, end-to-end times, $19$ prompts in each task after the warm-up sample. Table~\ref{tab:headline} uses length $5$.}
\label{tab:gammasweep}
\end{table}

\section{Reproducibility}
\label{app:repro}
Section~\ref{sec:method} and Appendix~\ref{app:impl} specify the drafter's architecture, training data, objective, and hyperparameters, and Appendix~\ref{app:algo} gives one decoding round as pseudocode. Appendix~\ref{app:theory} contains the complete proofs of all theoretical statements. Every measurement states its prompt count, and Appendix~\ref{app:ledger} lists each wall-clock comparison with its engine and sample size. \ifglancepublic Our decoding code, exactness audit, and law-fitting code are available at \url{https://github.com/js-lee-AI/GLANCE}, and we will release the drafter heads, the training code, and the round-level logs behind the entropy-law fits.\fi

\section{Additional Wall-Clock Results}
\label{app:ledger}

\paragraph{Where the cost sits.}
Across the five tasks of Table~\ref{tab:sglang}, prefill including the visual encoder takes $3$ to $21\%$ of the end-to-end time of autoregressive decoding, the most on DocVQA, whose answers are the shortest, and the decode loop takes the remainder. The share falls as the output grows, since prefill is paid once and the loop on every token. Counting prefill, GLANCE remains faster than EAGLE3-VL on the three lower-entropy tasks, by $7.6$ to $11.1\%$, and every paired bootstrap interval of these margins lies above zero.

\paragraph{Intervals behind the production-engine comparison.}
The margin column of Table~\ref{tab:sglang} is $\text{EAGLE3-VL}_{\text{ms}}/\text{GLANCE}_{\text{ms}}-1$, the ratio of the two printed speedups. Its interval is a paired bootstrap over the $101$ shared prompts of that statistic with $8000$ resamples. The intervals are $[-13.1,-10.6]$ on captioning, $[-13.1,-9.3]$ on TextVQA, $[+8.4,+12.8]$ on InfographicVQA, $[+8.5,+13.6]$ on DocVQA, and $[+8.7,+13.1]$ on ChartQA. GLANCE is the faster arm on $85$, $81$, and $90$ of the $101$ prompts of the three lower-entropy tasks. The same bootstrap, resampling the prompts of every task at once, puts the geometric-mean margin of $+1.3\%$ in $[+0.3,+2.2]$. Averaging prompt-level ratios instead of taking the ratio of means moves each figure by at most $0.9$ points and reorders nothing. The autoregressive arm stays within $2\%$ of its mean across the five tasks, so the tasks differ in what they generate rather than in what a token costs to decode.

\paragraph{EAGLE3-VL's released configuration.}
The model card of EAGLE3-VL lists an SGLang launch configuration of three draft steps, top-$2$ branching, and four draft tokens~\citep{aqmedai2025eagle3vl}. Run under the protocol of Table~\ref{tab:sglang}, back to back with the eight-step tree on one RTX A6000 for each task, it accepts $2.21$ to $2.71$ tokens a round against $3.51$ to $4.50$, and the eight-step tree decodes faster on every task, by $25.3\%$ on captioning, $32.9\%$ on TextVQA, $18.9\%$ on InfographicVQA, $21.0\%$ on DocVQA, and $30.0\%$ on ChartQA. Every paired bootstrap interval lies above zero, the tree is the faster arm on at least $98$ of the $101$ prompts of each task, and the geometric-mean margin is $+25.5\%$, with a $95\%$ interval from $24.6$ to $26.5\%$.

\paragraph{Training seeds.}
Two further heads were trained with the configuration of Table~\ref{tab:traincard} under seeds $1$ and $2$, with the target's fused states computed during training on two GPUs rather than precomputed, at the same global batch. Run under the protocol of Table~\ref{tab:sglang}, on one RTX A6000 for each task together with EAGLE3-VL and the primary head, both are faster than EAGLE3-VL on the three lower-entropy tasks, by $8.4$ and $9.0\%$ on InfographicVQA, $9.4$ and $9.2\%$ on DocVQA, and $9.0$ and $9.5\%$ on ChartQA, and every paired bootstrap interval lies above zero. On all five tasks their acceptance lengths lie within $4\%$ of the primary head's.

\paragraph{Longer contexts, tree against chain, and budget.}
Table~\ref{tab:ledger} lists the long-context comparisons, on which GLANCE is the faster arm on $31$ of the $44$ prompts at each context length. Table~\ref{tab:ledger-domain} times the tree against the chain, and Table~\ref{tab:budget} sweeps the verifier budget.

\begin{table}[t]
\centering
\setlength{\tabcolsep}{4pt}\footnotesize
\fitcol{%
\begin{tabular}{lcc}
\toprule
Comparison & Context & Result [$95\%$ CI] \\
\midrule
GLANCE vs.\ AR & $2$K & $2.30\times$ $[2.22,2.39]$ \\
EAGLE3-VL vs.\ AR & $2$K & $2.15\times$ $[2.07,2.24]$ \\
GLANCE vs.\ EAGLE3-VL & $2$K & $+6.9\%$ $[+3.4,+10.8]$ \\
\addlinespace
GLANCE vs.\ AR & ${\sim}6.7$K & $1.66\times$ $[1.61,1.73]$ \\
EAGLE3-VL vs.\ AR & ${\sim}6.7$K & $1.58\times$ $[1.51,1.64]$ \\
GLANCE vs.\ EAGLE3-VL & ${\sim}6.7$K & $+5.5\%$ $[+2.0,+9.1]$ \\
\bottomrule
\end{tabular}}
\caption{Decode-only wall-clock comparisons in SGLang on $44$ full-resolution InfographicVQA prompts, by context length. Percentages are the GLANCE-faster-by margin of Table~\ref{tab:sglang}, with paired bootstrap intervals over shared prompts.}
\label{tab:ledger}
\end{table}

\begin{table}[t]
\centering
\setlength{\tabcolsep}{4pt}\footnotesize
\fitcol{%
\begin{tabular}{lcccc}
\toprule
 & \multicolumn{3}{c}{ms/tok $\downarrow$} & \\
\cmidrule(lr){2-4}
Task & AR & tree & chain & tree vs.\ chain $\uparrow$ \\
\midrule
Captioning & $18.30$ & \g{$\mathbf{9.45}$} & $12.88$ & $1.36\times$ \\
TextVQA & $18.22$ & \g{$\mathbf{8.26}$} & $11.42$ & $1.38\times$ \\
DocVQA  & $18.15$ & \g{$\mathbf{8.55}$} & $11.33$ & $1.33\times$ \\
\midrule
Speedup over AR & $1.00\times$ & \g{$\mathbf{2.09\times}$} & $1.54\times$ & $\mathbf{1.36\times}$ \\
\bottomrule
\end{tabular}}
\caption{Decode time of the identical GLANCE head run as a budget-$63$ tree and as a width-$1$ chain in our Hugging Face implementation. The last row gives geometric means. Blue marks the tree.}
\label{tab:ledger-domain}
\end{table}

\begin{table*}[t]
\centering
\setlength{\tabcolsep}{5pt}\footnotesize
\fittext{%
\begin{tabular}{lcccc cccc}
\toprule
 & \multicolumn{4}{c}{acceptance length $\tau\ \uparrow$} & \multicolumn{4}{c}{speedup over AR $\uparrow$} \\
\cmidrule(lr){2-5}\cmidrule(lr){6-9}
Budget $\budget$ & $15$ & $31$ & $47$ & $63$ & $15$ & $31$ & $47$ & $63$ \\
\midrule
Captioning & $2.70$ & $2.89$ & $2.99$ & $3.04$ & $1.87$ & $1.99$ & $\mathbf{2.05}$ & $\mathbf{2.05}$ \\
TextVQA & $2.95$ & $3.14$ & $3.25$ & $3.29$ & $1.95$ & $2.21$ & $2.24$ & $\mathbf{2.25}$ \\
InfographicVQA & $3.21$ & $3.43$ & $3.64$ & $3.63$ & $2.22$ & $2.39$ & $\mathbf{2.46}$ & $2.39$ \\
DocVQA & $3.35$ & $3.61$ & $3.68$ & $3.78$ & $2.31$ & $2.49$ & $2.51$ & $\mathbf{2.55}$ \\
ChartQA & $4.43$ & $4.81$ & $5.00$ & $5.16$ & $3.07$ & $3.30$ & $3.40$ & $\mathbf{3.42}$ \\
\bottomrule
\end{tabular}}
\caption{Verifier-budget sweep of GLANCE in our Hugging Face implementation. All four budgets of a task run in one process and share its autoregressive baseline.}
\label{tab:budget}
\end{table*}

\paragraph{Round costs.}
A GLANCE round is one draft pass and one verify pass, and tree width enters only inside the verify. Regressing the production head's round time in SGLang on its number of draft passes, over five configurations from $4$ to $48$ draft tokens at $3$ to $10$ steps on captioning, gives $24.8$\,ms plus $1.82$\,ms for every sequential draft pass. At batch one the verify is therefore nearly flat in tree width, while depth is paid one pass at a time. Widening our tree $4.2\times$ costs $3$ to $5\%$ of a round, whereas taking that head from two passes to eight costs $38\%$ of one.

\section{Classic Two-Model Speculative Decoding}
\label{app:classicsd}
Classic two-model speculation runs on our target through the official Hugging Face assisted-generation path with three drafts, the image-conditioned Qwen3-VL-4B and Qwen3-VL-2B and a text-only Qwen3-1.7B in the spirit of the strong baseline of the first multimodal study~\citep{gagrani2024speculative}. Table~\ref{tab:classicsd} reports all three. The $4$B draft accepts long blocks, with $\tau$ from $3.53$ to $4.47$ rising with grounding, yet decodes at $0.61$ to $0.80\times$, since a half-size draft costs roughly half a target forward for each drafted token and no acceptance amortizes that. The $2$B draft accepts less and stays below $1\times$ throughout, so halving the draft again narrows the deficit without closing it. The text-only draft shows what vision access is worth, since its acceptance collapses to $1.42$ to $2.12$. Unlike the text-only states of Figure~\ref{fig:systems}, which the block head still reads, this external draft reads nothing of the target at all.

\begin{table*}[t]
\centering
\setlength{\tabcolsep}{5pt}\footnotesize
\begin{tabular}{llccc}
\toprule
Draft & Task & $n$ & $\tau\,\uparrow$ [$95\%$ CI] & Speedup over AR\,$\uparrow$ [CI] \\
\midrule
Qwen3-VL-4B & Captioning & $100$ & $3.53$ [$3.42,3.64$] & $0.61\times$ [$0.60,0.63$] \\
            & TextVQA    & $100$ & $3.79$ [$3.65,3.93$] & $0.67\times$ [$0.65,0.69$] \\
            & InfoVQA    & $100$ & $3.95$ [$3.74,4.17$] & $0.70\times$ [$0.68,0.72$] \\
            & DocVQA     & $100$ & $4.39$ [$4.19,4.60$] & $0.76\times$ [$0.74,0.78$] \\
            & ChartQA    & $100$ & $4.47$ [$4.31,4.63$] & $0.80\times$ [$0.79,0.82$] \\
\addlinespace
Qwen3-VL-2B & Captioning & $100$ & $3.11$ [$3.01,3.21$] & $0.66\times$ [$0.64,0.68$] \\
            & TextVQA    & $100$ & $3.26$ [$3.15,3.39$] & $0.72\times$ [$0.70,0.74$] \\
            & InfoVQA    & $100$ & $3.21$ [$3.07,3.35$] & $0.74\times$ [$0.72,0.76$] \\
            & DocVQA     & $100$ & $3.58$ [$3.43,3.74$] & $0.84\times$ [$0.82,0.87$] \\
            & ChartQA    & $100$ & $3.97$ [$3.81,4.13$] & $0.90\times$ [$0.88,0.92$] \\
\addlinespace
Qwen3-1.7B, text-only & Captioning & $100$ & $1.49$ [$1.46,1.51$] & $0.36\times$ [$0.35,0.38$] \\
            & TextVQA    & $100$ & $1.42$ [$1.40,1.44$] & $0.32\times$ [$0.31,0.34$] \\
            & InfoVQA    & $100$ & $1.90$ [$1.84,1.97$] & $0.47\times$ [$0.45,0.49$] \\
            & DocVQA     & $100$ & $1.68$ [$1.64,1.72$] & $0.43\times$ [$0.41,0.45$] \\
            & ChartQA    & $100$ & $2.12$ [$2.05,2.20$] & $0.58\times$ [$0.56,0.61$] \\
\bottomrule
\end{tabular}
\caption{Classic two-model speculative decoding on Qwen3-VL-8B with three drafts, all on the same prompts and protocol, with speedups from the same run.}
\label{tab:classicsd}
\end{table*}

\section{Additional Law and Transfer Results}
\label{app:extra}

\paragraph{Fits on all five tasks.}
Table~\ref{tab:lawfit} lists the entropy-acceptance fits, and Figure~\ref{fig:lawcurves} draws them. InfoVQA is logged under the same streaming protocol as captioning, TextVQA, and DocVQA, while ChartQA is logged under its templated evaluation prompts and carries the steepest slope of the five. Both floors clear their own caps, $4.42$ against $3.22$ and $6.92$ against $4.50$, and the share of rounds accepting eight or more tokens rises with grounding, at $0.0$, $0.9$, $2.6$, $4.4$, and $11.5\%$ from captioning to ChartQA.

\begin{table*}[t]
\centering
\setlength{\tabcolsep}{4pt}\footnotesize
\begin{tabular}{lcccccccc}
\toprule
Task & Rounds & $b_0$ & $b_1$ & $R^2_{\mathrm{par}}$ & $R^2_{\mathrm{iso}}$ & $R^2_{\mathrm{rnd}}$ & Ceiling fraction & Mean $\acc$ \\
\midrule
Captioning & $4026$ & $0.819$ & $0.459$ & $0.949$ & $\mathbf{0.990}$ & $0.097$ & $89\%$ & $1.881$ \\
TextVQA & $2507$ & $0.936$ & $0.578$ & $0.704$ & $\mathbf{0.968}$ & $0.091$ & $63\%$ & $2.120$ \\
InfoVQA & $1858$ & $1.169$ & $0.804$ & $0.679$ & $\mathbf{0.979}$ & $0.128$ & $62\%$ & $2.672$ \\
DocVQA  & $1218$ & $1.226$ & $0.954$ & $0.452$ & $\mathbf{0.964}$ & $0.074$ & $39\%$ & $3.076$ \\
ChartQA, templated & $2190$ & $1.504$ & $1.710$ & $0.543$ & $\mathbf{0.991}$ & $0.138$ & $50\%$ & $3.709$ \\
\midrule
Pooled, captioning, TextVQA, DocVQA & $7751$ & $0.964$ & $0.630$ & $0.705$ & $0.917$ & $0.107$ & $62\%$ & $2.146$ \\
\bottomrule
\end{tabular}
\caption{Entropy-acceptance fits, with the affine-logit coefficients, the parametric, isotonic, and round-level $R^2$, the fraction of the entropy-only ceiling realized, and the mean accepted length.}
\label{tab:lawfit}
\end{table*}

\begin{figure}[t]
\centering
\includegraphics[width=\columnwidth]{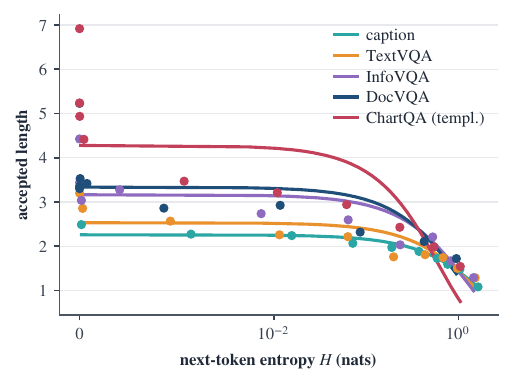}
\caption{Fitted entropy-acceptance curves for each task.}
\label{fig:lawcurves}
\end{figure}

\paragraph{Brackets.}
Figure~\ref{fig:boundbody} places each task's fitted law between two independent bounds, the AdaEDL lower bound~\citep{agrawal2024adaedl} and the verification ceiling $2\log\budget/\HH$ at $\budget{=}31$, and every fit stays inside both.

\begin{figure*}[t]
\centering
\includegraphics[width=\textwidth]{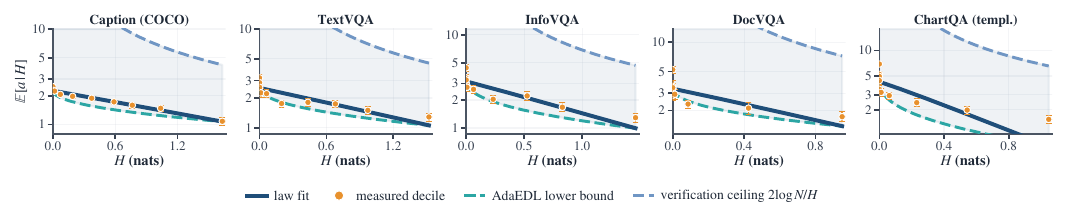}
\caption{Fitted law of each task, drawn against root entropy $\HH$ on a log scale, with the AdaEDL lower bound, the verification ceiling $2\log\budget/\HH$ at $\budget{=}31$, and measured entropy deciles as points.}
\label{fig:boundbody}
\end{figure*}

\paragraph{Entropy against image saliency.}
Table~\ref{tab:charstats} gives the correlations and decile gaps behind the law. Entropy correlates with acceptance on every task, while an image-saliency signal $v_t$ does not survive conditioning on entropy, with partial $|r|\le0.09$ and gaps near $1$ within entropy quintiles.

\begin{table}[t]
\centering
\setlength{\tabcolsep}{2pt}\footnotesize
\fitcol{%
\begin{tabular}{lccc}
\toprule
Quantity & Captioning & TextVQA & DocVQA \\
\midrule
Spearman $\rho(\HH,\acc)$ & $-0.311$ & $-0.351$ & $-0.384$ \\
\quad $p$ & $3.8{\times}10^{-91}$ & $1.5{\times}10^{-73}$ & $4.0{\times}10^{-44}$ \\
\quad low/high $\HH$-decile gap & $2.30\times$ & $2.49\times$ & $3.04\times$ \\
\addlinespace
Spearman $\rho(v_t,\acc)$ & $-0.098$ & $-0.025$ & $+0.129$ \\
\quad partial $r(v_t,\acc\mid\HH)$ & $-0.058$ & $-0.026$ & $+0.088$ \\
\quad within-quintile $v_t$ gap & $1.07\times$ & $1.07\times$ & $0.94\times$ \\
\bottomrule
\end{tabular}}
\caption{Correlations and gaps for accepted length $\acc$, entropy $\HH$, and image saliency $v_t$. The within-quintile gap splits each entropy quintile at its median $v_t$.}
\label{tab:charstats}
\end{table}

\paragraph{Mismatched-image probe.}
The correct image lengthens accepted blocks by $+3.5\%$ on captioning with a $95\%$ interval of $[+1.5,+5.7]$, by $+6.8\%$ on TextVQA with $[+4.2,+9.4]$, and by $+11.7\%$ on DocVQA with $[+7.7,+16.3]$. Once entropy is held fixed, the pooled coefficient of the residual image effect is $-0.07$.

\paragraph{Transfer across drafters, targets, and modalities.}
Table~\ref{tab:lawrobust} varies the target, and Figure~\ref{fig:crosstarget} plots the fits. The two further Qwen targets use same-family two-model drafts, which are autoregressive, so the form holds for both kinds of drafter and is a property of the target. On Llama-3.2-11B-Vision, where vision enters by cross-attention rather than as a token prefix, GLANCE remains lossless and the law holds on captioning, and we state the characterization for prefix-fusion targets. Table~\ref{tab:crossmod} extends the form beyond vision to speech recognition~\citep{xu2025qwen25omni}, code~\citep{hui2024qwen25coder}, chat~\citep{zhao2024wildchat}, and a non-Qwen backbone~\citep{olmo2024olmo2}. Grounded vision-language generation and code follow the determinate-copy route of the two-state analysis, with long verbatim runs, while speech behaves like a uniform lift at every position, and both are limits of the same survival family.

\begin{table}[t]
\centering
\setlength{\tabcolsep}{5pt}\footnotesize
\fitcol{%
\begin{tabular}{lccc}
\toprule
Target & Qwen2.5-VL & Qwen2-VL & Llama-3.2 \\
 & 7B & 7B & 11B-Vision \\
\midrule
Drafter & two-model & two-model & GLANCE \\
$b_1$ & $0.67$ & $0.40$ & $0.41$ \\
Curve $R^2\uparrow$ & $0.70$ & $0.95$ & $0.91$ \\
$\rho(\HH,\acc)$ & $-0.42$ & $-0.30$ & $-0.38$ \\
Mean $\acc\uparrow$ & $3.67$ & $3.08$ & $2.86$ \\
\bottomrule
\end{tabular}}
\caption{Transfer of the law's form across targets. The two-model drafts are same-family $3$B and $2$B models, and the Llama column is fitted at horizon $\Lmax{=}9$ on its captioning rounds. Mean $\acc$ excludes the bonus token.}
\label{tab:lawrobust}
\end{table}

\begin{table}[t]
\centering
\setlength{\tabcolsep}{3pt}\footnotesize
\fitcol{%
\begin{tabular}{llccc}
\toprule
Setting & Model pair & $b_0$ & $b_1$ & $\tau\,\uparrow$ \\
\midrule
speech (ASR)      & Qwen2.5-Omni 7B$+$3B    & $2.07$ & $0.49$ & $5.41$ \\
code              & Qwen2.5-Coder 7B$+$0.5B & $1.94$ & $1.76$ & $5.77$ \\
chat (WildChat)   & two-model               & $1.66$ & $0.64$ & $4.33$ \\
non-Qwen backbone & OLMo-2 13B$+$7B         & $1.85$ & $0.84$ & $5.25$ \\
\bottomrule
\end{tabular}}
\caption{Form of the law across modalities and on a non-Qwen backbone, with two-model pairs.}
\label{tab:crossmod}
\end{table}

\begin{figure}[t]
\centering
\includegraphics[width=\columnwidth]{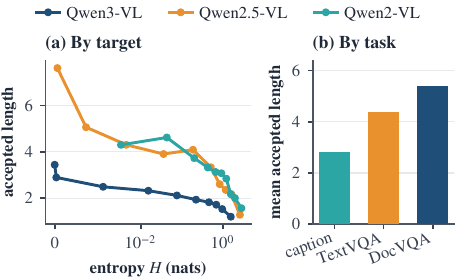}
\caption{Entropy-acceptance fits for three targets (a) and the mean accepted length of each task on Qwen2.5-VL-7B (b).}
\label{fig:crosstarget}
\end{figure}

\paragraph{Conditioning ablation.}
Table~\ref{tab:ablation} tabulates the values behind panel (b) of Figure~\ref{fig:systems} together with the zeroed-state control. Pooled over the five tasks and their $300$ prompts, fused visual conditioning lifts acceptance by $+14.1\%$, with a $95\%$ interval of $[+12.8,+15.5]\%$, and the grounded tasks carry the lift.

\begin{table}[t]
\centering
\setlength{\tabcolsep}{4pt}\small
\fitcol{%
\begin{tabular}{lcccc}
\toprule
 & \multicolumn{3}{c}{$\tau\uparrow$ by what the head reads} & \\
\cmidrule(lr){2-4}
Task & fused VL & text-only & zeroed & text/fused \\
\midrule
Captioning & \g{$\mathbf{2.92}$} & $2.84$ & $1.10$ & $0.97$ \\
TextVQA    & \g{$\mathbf{3.22}$} & $2.95$ & $1.12$ & $0.92$ \\
InfoVQA    & \g{$\mathbf{3.51}$} & $3.12$ & $1.10$ & $0.89$ \\
DocVQA     & \g{$\mathbf{3.62}$} & $2.91$ & $1.10$ & $0.80$ \\
ChartQA    & \g{$\mathbf{4.80}$} & $4.02$ & $1.15$ & $0.84$ \\
\midrule
Mean       & \g{$\mathbf{3.61}$} & $3.17$ & $1.11$ & $0.88$ \\
\bottomrule
\end{tabular}}
\caption{Conditioning ablation at budget $31$ with $60$ prompts in each task. The text-only column uses a text-only Qwen3-8B's states at text positions and zeros at image positions. Blue marks GLANCE.}
\label{tab:ablation}
\end{table}

\end{document}